\documentclass[twoside,11pt]{article}

\usepackage{amssymb, amsmath, amsthm}
\usepackage{epsfig}
\usepackage{array}
\usepackage{ifthen}
\usepackage{color}
\usepackage{fancyhdr}
\usepackage{graphicx}
\usepackage{mathtools}
\usepackage{csquotes}
\usepackage{multirow}
\usepackage{xcolor}
\usepackage{chngcntr}
\usepackage{apptools}
\usepackage[export]{adjustbox}
\usepackage{hyperref}
\AtAppendix{\counterwithin{lemma}{section}}
\usepackage[comma,authoryear]{natbib}
\newcommand{\E}{\textbf{E}}

\newcommand{\argmax}{\text{argmax}}

\newcommand{\argmin}{\text{argmin}}

\newcommand{\comm}[1]{}

\newenvironment{myfont}{\fontfamily{phv}\selectfont}{\par}
\newtheorem{definition}{Definition}
\newtheorem{theorem}{Theorem}
\newtheorem{lemma}{Lemma}

\begin{document}

\title{Extrapolating Expected Accuracies for Large Multi-Class Problems}

%\editor{}

\author{ Charles Zheng\\ {\tt charles.y.zheng@gmail.com} \\
       Rakesh Achanta \\ {\tt rakesha@stanford.edu} \\
       Yuval Benjamini \\ {\tt yuval.benjamini@mail.huji.ac.il}}

\maketitle

\begin{abstract}%   <- trailing '%' for backward compatibility of .sty file
The difficulty of multi-class classification generally increases with
the number of classes.  Using data from a subset of the classes, can
we predict how well a classifier will scale with an increased number
of classes?  Under the assumptions that the classes are sampled
identically and independently from a population, and that the classifier is based on independently learned
scoring functions, we show that the expected accuracy when the
classifier is trained on $k$ classes is the $k-1$st moment of a
certain distribution that can be estimated from data.  We present an
unbiased estimation method based on the theory, and demonstrate its
application on a facial recognition example.
\end{abstract}

%\begin{keywords}
%Multiclass problems, face recognition, object recognition, transfer learning, nonparametric models
%\end{keywords}
\section{Introduction}\label{sec:recog_tasks}

Many machine learning tasks are interested in recognizing or
identifying an individual instance within a large set of possible
candidates. These problems are usually modeled as multi-class
classification problems, with a large and possibly complex label
set. Leading examples include detecting the speaker from his voice
patterns \citep{togneri2011overview}, identifying the author from her
written text \citep{stamatatos2014overview}, or labeling the object
category from its image
\citep{duygulu2002object,deng2010does,oquab2014learning}.  In all
these examples, the algorithm observes an input $x$, and uses the
classifier function $h$ to guess the label $y$ from a large label set
$\mathcal{S}$.

There are multiple practical challenges in developing classifiers for large label sets. Collecting high quality training data is
perhaps the main obstacle, as the costs scale with the number of
classes.  It can be affordable to first collect data for a small set
of classes, even if the long-term goal is to generalize to a larger
set.  Furthermore, classifier development can be accelerated by
training first on fewer classes, as each training cycle may require
substantially less resources.  Indeed, due to interest in how
small-set performance generalizes to larger sets, such comparisons can
found in the literature \citep{oquab2014learning, griffin2007caltech}.
A natural question is: how does changing the size of the label set
affect the classification accuracy?

We consider a pair of classification problems on finite
label sets: a source task with label set $\mathcal{S}_{k_1}$ of size
$k_1$, and a target task with a larger label set $\mathcal{S}_{k_2}$
of size $k_2 > k_1$.  For each label set $\mathcal{S}_k$, one
constructs the classification rule $h^{(k)}:\mathcal{X} \to
\mathcal{S}_{k}$.  Supposing that in each task, the test example
$(X^*, Y^*)$ has a joint distribution, define the generalization
accuracy for label set $\mathcal{S}_k$ as
\begin{equation}\label{eq:ga_k}
\text{GA}_k = \Pr[h^{(k)}(X^*) = Y^*].
\end{equation}
The problem of \emph{performance extrapolation} is the following:
using data from only the source task $\mathcal{S}_{k_1}$, 
predict the accuracy for a target task with a larger unobserved 
label set $\mathcal{S}_{k_2}$.

A natural use case for performance extrapolation would be in the
deployment of a facial recognition system.  Suppose a system was
developed in the lab on a database of $k_1$ individuals. Clients would
like to deploy this system on a new larger set of $k_2$
individuals. Performance extrapolation could allow the lab to predict
how well the algorithm will perform on the client's problem,
accounting for the difference in label set size.

Extrapolation should be possible when the source and target
classifications belong to the same problem domain.  In
many cases, the set of categories $\mathcal{S}$ is to some degree a
random or arbitrary selection out of a larger, perhaps infinite, set
of potential categories $\mathcal{Y}$. Yet any specific experiment
uses a fixed finite set.  For example, categories in the classical
Caltech-256 image recognition data set \citep{griffin2007caltech} were
assembled by aggregating keywords proposed by students and then
collecting matching images from the web.  The arbitrary nature of the
label set is even more apparent in biometric applications (face
recognition, authorship, fingerprint identification) where the labels
correspond to human individuals \citep{togneri2011overview,
  stamatatos2014overview}.  In all these cases, the number of the
labels used to define a concrete data set is therefore an experimental
choice rather than a property of the domain.  Despite the arbitrary
nature of these choices, such data sets are viewed as representing the
larger problem of recognition within the given domain, in the sense
that success on such a data set should inform performance on similar
problems.

In this paper, we assume that both $\mathcal{S}_{k_1}$ and
$\mathcal{S}_{k_2}$ are independent identically distributed (i.i.d.)
samples from a population (or prior distribution) of labels $ \pi$, which is
defined on the label space $\mathcal{Y}$.  These assumptions help
concretely analyze the generalization accuracy, although both are only
approximate characterizations of the label selection process, which is
often at least partially manual. Since we assume the label set is
random, the generalization accuracy of a given classifier becomes a
random variable.  Performance extrapolation then becomes the problem
of estimating the average generalization accuracy $\text{AGA}_k$ of an
i.i.d. label set $\mathcal{S}_k$ of size $k$.  The condition of
i.i.d. sampling of labels ensures that the separation of labels in a
random set $\mathcal{S}_{k_2}$ can be inferred by looking at the
empirical separation in $\mathcal{S}_{k_1}$, and therefore that some
estimate of the average accuracy on $\mathcal{S}_{k_2}$ can be
obtained.  We also make the assumption that the classifiers train a
separate model for each class.  This convenient property allows us to
characterize the accuracy of the classifier by selectively
conditioning on one class at a time.

Our paper presents two main contributions related to extrapolation.
First, we present a theoretical formula describing how average accuracy
for smaller $k$ is linked to average accuracy for label set of size
$K>k$.  We show that accuracy at any size depends on a
\emph{discriminability} function ${D}$, which is determined by properties
of the data distribution and the classifier.  Second, we propose an
estimation procedure that allows extrapolation of the observed average
accuracy curve from $k_1$-class data to a larger number of classes,
based on the theoretical formula. Under certain conditions, the
estimation method has the property of being an unbiased estimator of
the average accuracy.

The paper is organized as follows.  In the rest of this section, we
discuss related work.  The framework of randomized classification is
introduced in Section \ref{sec:rc_motivation}, and there we also
introduce a toy example which is revisited throughout the
paper. Section \ref{sec:extrapolation} develops our theory of
extrapolation, and Section \ref{sec:extrapolation_estimation} we
suggest an estimation method. We evaluate our method using simulations in Section 4.
In Section
\ref{sec:extrapolation_example}, we demonstrate our method on a facial
recognition problem, as well as an optical character recognition problem. In Section \ref{sec:discussion} we
discuss modeling choices and limitations of our theory, as well as
potential extensions.

\subsection{Related Work}

Linking performance between two different but related classification
tasks can be considered an instance of transfer learning
\citep{pan2010survey}. Under \citeauthor{pan2010survey}'s terminology,
our setup is an example of multi-task learning, because the source
task has labeled data, which is used to predict performance on a
target task that also has labeled data. Applied examples of transfer learning from one label set to another include \cite{oquab2014learning},
\cite{donahue2014decaf}, \cite{sharif2014cnn}. 
However, there is little theory for predicting the behavior of the learned classifier on a new label set. Instead, most research classification for large label sets deal with the computational challenges of jointly optimizing the many parameters
required for these models for specific classification algorithms \citep{crammer2001algorithmic,
  lee2004multicategory, weston1999support}. \cite{gupta2014training}
presents a method for estimating the accuracy of a classifier which can
be used to improve performance for general classifiers, but doesn't apply for different set sizes.  

The theoretical framework we adopt is one where there exists a family
of classification problems with increasing number of classes. This
framework can be traced back to \cite{Shannon1948}, who considered the error rate of a random codebook, which is a special case of randomized classification. More recently, a number of authors have considered the problem of high-dimensional feature selection for multiclass
classification with a large number of classes \citep{pan2016ultrahigh,
  abramovich2015feature, davis2011bayesian}.  All of these works
assume specific distributional models for classification compared to
our more general setup. However, we do not deal with the problem of
feature selection.

Perhaps the most similar method that deals with extrapolation of classification error to a larger number of classes can be found in \cite{Kay2008a}. They trained a classifier for identifying the observed stimulus from a functional MRI scan of brain activity, and were interested in its performance on larger stimuli sets. They proposed an extrapolation algorithm as a heuristic with little theoretical discussion. In Section \ref{sec:KDEcomparison} we interpret their method within our theory, and discuss cases where it performs well compared to our algorithm.

\section{Randomized Classification}\label{sec:rc_motivation}

The randomized classification model we study has the following
features.  We assume that there exists an infinite, perhaps
continuous, label space $\mathcal{Y}$ and a example space $\mathcal{X}
\in \mathbb{R}^p$.  We assume there exists a prior distribution $\pi$
on the label space $\mathcal{Y}$.  And for each label $y \in
\mathcal{Y}$, there exists a distribution of examples $F_y$. In other
words, for an example-label pair $(X, Y)$, the conditional
distribution of $X$ given $Y = y$ is given by $F_y$.
%Furthermore, we assume that
%there exists a prior distribution $\pi$ on the label space $\mathcal{Y}$.

A random classification task can be generated as follows.  The label
set $\mathcal{S} = \{Y^{(1)},\hdots, Y^{(k)}\}$ is generated by
drawing labels $Y^{(1)},\hdots, Y^{(k)}$ i.i.d. from $\pi$.  For each
label, we sample a training set and a test set.  The training set is
obtained by sampling $r_{train}$ observations $X_{j, train}^{(i)}$
i.i.d. from $F_{Y^{(i)}}$ for $j = 1,\hdots, r_{train}$ and $i =
1,\hdots, k$.  The test set is likewise obtained by sampling $r$
observations $X_j^{(i)}$ i.i.d. from $F_{Y^{(i)}}$ for $j = 1,\hdots,
r$.

We assume that the classifier $h(x)$ works by assigning a score to
each label $y^{(i)} \in \mathcal{S}$, then choosing the label with the
highest score.  That is, there exist real-valued \emph{score
  functions} $m_{y^{(i)}}(x)$ for each label $y^{(i)} \in
\mathcal{S}$.  Since the classifier is allowed to depend on the
training data, it is convenient to view it (and its associated score
functions) as random.  We write $H(x)$ when we wish to work with the
classifier as a random function, and likewise $M_y(x)$ to denote the
score functions whenever they are considered as random.

For a fixed instance of the classification task with labels
$\mathcal{S} = \{y^{(i)}\}_{i=1}^k$ and associated score functions
$\{m_{y^{(i)}}\}_{i=1}^k$, recall the definition of the $k$-class
generalization error \eqref{eq:ga_k}.  Assuming that there are no
ties, it can be written in terms of score functions as
\[
\text{GA}_k(h) = \frac{1}{k} \sum_{i=1}^k  \Pr[m_{y^{(i)}}(X^{(i)}) = \max_j
m_{y^{(j)}}(X^{(i)})],
\]
where $X^{(i)} \sim F_{y^{(i)}}$ for $i =1,\hdots, k$.  However, when
we consider the labels $\{Y^{(i)}\}_{i=1}^k$ and associated score
functions to be random, the generalization accuracy also becomes a
random variable.

Suppose we specify $k$ but do not fix any of the random quantities in
the classification task.  Then the $k$-class \emph{average
  generalization accuracy} of a classifier is the expected value of
the generalization accuracy $\text{GA}_k(H)$ resulting from a random
set of $k$ labels, $Y^{(1)}, \hdots, Y^{(k)} \stackrel{iid}{\sim
  \pi}$, and their associated score functions,
\begin{align*}
\text{AGA}_k &= \frac{1}{k} \sum_{i=1}^k \Pr[M_{Y^{(i)}}(X^{(i)}) = \max_j
M_{Y^{(j)}}(X^{(i)})]
\\&= \Pr[M_{Y^{(1)}}(X^{(1)}) = \max_j M_{Y^{(j)}}(X^{(1)})].
\end{align*}
The last line follows from noting that all $k$ summands in the
previous line are identical.  The definition of average
  generalization accuracy is illustrated in Figure
  \ref{fig:average_risk}.

\begin{figure}[t]
\centering
\includegraphics[scale = 0.3]{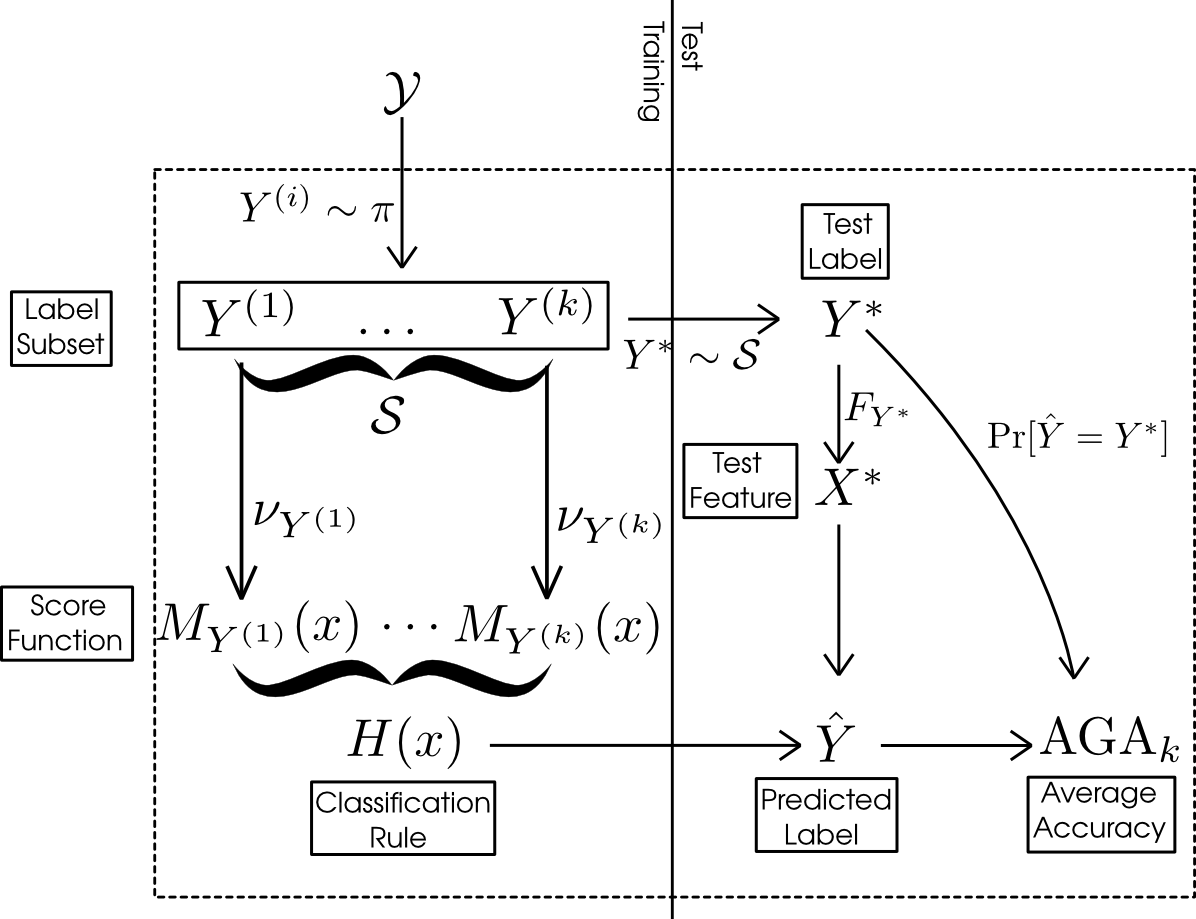}
\caption{\textbf{Average generalization accuracy:} A diagram of the random quantities underlying the average generalization accuracy for $k$ labels ($\text{AGA}_k$). At the training stage (left), a set of $k$ labels $\mathcal{S}$ is sampled from the prior $\pi$, and score functions are trained from examples for these classes. At the test stage (right), one true class $Y^*$ is sampled uniformly from $\mathcal{S}$, as well as a test example $X^*$. $\text{AGA}_k$ measures the expected accuracy over these random variables.}\label{fig:average_risk}
\end{figure}

\subsection{Marginal Classifier}

In our analysis, we do not want the classifier to rely too strongly on
complicated interactions between the labels in the set. We therefore
propose the following property of marginal separability for
classification models:

\begin{definition}
The classifier $H(x)$ is called a \emph{marginal classifier} if the
score function $M_{y^{(i)}}(x)$ only depends on the label $y^{(i)}$
and the class training set $X_{j, train}^{(i)}$; that is, for some function $g$,
\[M_{y^{(i)}}(x) = g(x; y^{(i)},X_{1, train}^{(i)},...,X_{r_{train}, train}^{(i)}).\]
\end{definition}
This means that the score function for $y^{(i)}$ does not depend on
other labels $y^{(j)}$ or their training samples.  Therefore, each
$M_y$ can be considered to have been drawn from a distribution
$\nu_y$.  Classes ``compete'' only through selecting the highest
score, but not in constructing the score functions.  The operation of
a marginal classifier is illustrated in Figure
\ref{fig:classification_rule}.

\begin{figure}[t]
\centering
\includegraphics[scale = 0.4]{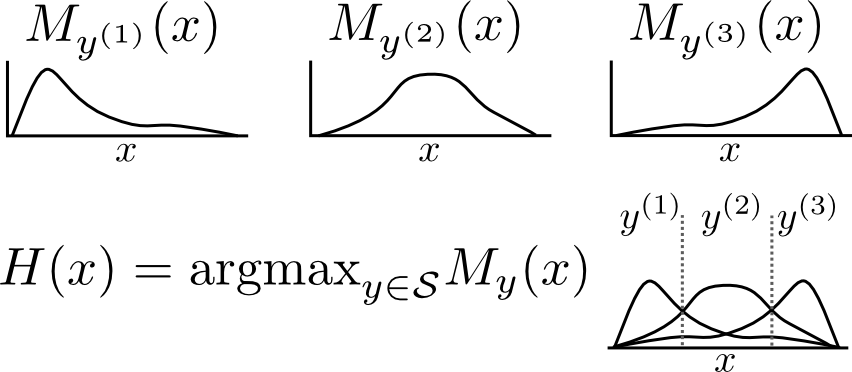}
\caption{\textbf{Classification rule:} Top: Score functions for three classes in a one-dimensional example space. Bottom: The classification rule chooses between $y^{(1)},y^{(2)}$ or $y^{(3)}$ by choosing the maximal score function. 
%A classifier is marginal if each score function does not depend on labels or training samples for other classes. 
}
\label{fig:classification_rule}
\end{figure}

The \emph{marginal} property allows us to prove
strong results about the accuracy of the classifier under
i.i.d. sampling assumptions.

\textbf{Comments:}
\begin{enumerate}
\item If $H$ is a marginal classifier then 
$M_{Y^{(i)}}$ is independent of $Y^{(j)}$ and $M_{Y^{(j)}}$ for $i \neq j$.
\item Estimated Bayes classifiers are primary examples of marginal
  classifiers. Let $\hat{f_y}$ be a density estimate of the example
  distribution under label $y$ obtained from the empirical
  distribution $\hat{F_y}$. Then, we can use the estimated density to
  produce the score functions:
\[ M^{EB}_y(x) = \log(\hat{f_{y}}(x)).\]
The resulting empirical approximation for the Bayes classifier would
be
\[ H^{EB}(x) = \text{argmax}_{Y \in \mathcal{S}}(M^{EB}_Y(x)).\]
\item Both Quadratic Discriminant Analysis (QDA) and na\:{i}ve Bayes
  classifiers can be seen as specific instances of an estimated Bayes
  classifier.
\footnote{QDA is the special case of the estimated Bayes classifier
  when $\hat{f_y}$ is obtained as the multivariate Gaussian density
  with mean and covariance parameters estimated from the data.  Naive
  Bayes is the estimated Bayes classifier when $\hat{f_y}$ is obtained
  as the product of estimated componentwise marginal distributions of
  $p(x_i|y)$.}  For QDA, the score function is given by
\[
m_y^{QDA}(x) = -(x - \mu(\hat{F}_y))^T \Sigma(\hat{F}_y)^{-1} (x-\mu(\hat{F}_y)) - \log\det(\Sigma(\hat{F}_y)),
\]
where $\mu(F) = \int y dF(y)$ and $\Sigma(F) = \int (y-\mu(F))(y-\mu(F))^T dF(y)$.
In Naive Bayes, the score function is
\[
m^{NB}_y(x) = \sum_{j=1}^p \log \hat{f}_{y, j}(x),
\]
where $\hat{f}_{y, j}$ is a density estimate for the $j$-th component of
$\hat{F}_y$.
\item For some classifiers, $M_y$ is a deterministic function of $y$
  (and therefore $\nu_y$ is degenerate). A prime example is when
  there exist fixed or pre-trained embeddings $g, \tilde{g}$ that map
  labels $y$ and examples $x$ into $R^p$. Then
\begin{equation}
M_y^{embed} = -\|g(y) - \tilde{g}(x)\|_2.
\end{equation}
\item There are many classifiers which do not satisfy the marginal
  property, such as multinomial logistic regression, multilayer neural
  networks, decision trees, and k-nearest neighbors.
\end{enumerate}

\emph{Notational remark.}  Henceforth, we shall relax the assumption
that the classifier $H(x)$ is based on a training set.  Instead, we
assume that there exist score functions $\{M_{Y^{(i)}}\}_{i=1}^k$
associated with the random label set $\{Y^{(i)}\}_{i=1}^k$, and that
the score functions $M_{Y^{(i)}}$ are independent of the test set.
The classifier $H(x)$ is marginal if and only if $M_{Y^{(i)}}$ are
independent of both $Y^{(j)}$ and $M_{Y^{(j)}}$ for $j \neq i$.

\subsection{Estimation of Average Accuracy}\label{sec:estimation_average_accuracy}
Before tackling extrapolation, it is useful to discuss a simpler task of generalizing accuracy results when the target set is \emph{not} larger than the source set. Suppose we have test data for a classification task with $k_1$
classes.  That is, we have a label set $\mathcal{S}_{k_1} =
\{y^{(i)}\}_{i=1}^{k_1}$ and its associated set of score functions
$M_{y^{(i)}}$, as well as test observations $(x_1^{(i)},\hdots,
x_{r}^{(i)})$ for $i = 1,\hdots, k_1$.  What would be the predicted
accuracy for a new randomly sampled set of $k_2 \leq k_1$ labels?

Note that $\text{AGA}_{k_2}$ is the expected value of the accuracy on
the new set of $k_2$ labels.  Therefore, any unbiased estimator of
$\text{AGA}_{k_2}$ will be an unbiased predictor for the accuracy on
the new set.

Let us start with the case $k_2 = k_1 = k$.  For each test observation
$x_j^{(i)}$, define the ranks of the candidate classes $\ell =
1,\hdots, k$ by
\[
R_{j}^{i, \ell} = \sum_{s = 1}^k I\{m_{y^{(\ell)}}(x_j^{(i)}) \geq m_{y^{(s)}}(x_j^{(i)})\}.
\]
The test accuracy is the fraction of observations for which the
correct class also has the highest rank
\begin{equation}\label{eq:test_risk}
\text{TA}_k = \frac{1}{r k} \sum_{i=1}^{k} \sum_{j=1}^{r} I\{R_j^{i,i} = k\}.
\end{equation}
Taking expectations over both the test set and the random labels, the
expected value of the test accuracy is $\text{AGA}_k$.  Therefore, in this special case, $\text{TA}_k$ provides an unbiased estimator for $\text{AGA}_{k_2}$.

Next, let us consider the case where $k_2 < k_1$.  Consider label set
$\mathcal{S}_{k_2}$ obtained by sampling $k_2$ labels uniformly
without replacement from $\mathcal{S}_{k_1}$. Since
$\mathcal{S}_{k_2}$ is unconditionally an i.i.d. sample from the
population of labels $\pi$, the test accuracy of $\mathcal{S}_{k_2}$
is an unbiased estimator of $\text{AGA}_{k_2}$.  However, we can get a
better unbiased estimate of $\text{AGA}_{k_2}$ by averaging over all
the possible subsamples $\mathcal{S}_{k_2} \subset \mathcal{S}_{k_1}$.
This defines the average test accuracy over subsampled tasks,
$\text{ATA}_{k_2}$.

\emph{Remark.}  Na\"{i}vely, computing $\text{ATA}_{k_2}$ requires us
to train and evaluate ${k_1}\choose{k_2}$ classification rules.
However, for marginal classifiers, retraining the classifier is not
necessary.  Looking at the rank $R_{j}^{i,i}$ of the correct label $i$
for $x_j^{(i)}$, allows us to determine how many subsets
$\mathcal{S}_2$ will result in a correct classification. Specifically,
there are $R_{j}^{i,i} - 1$ labels with a lower score than the correct
label $i$.  Therefore, as long as one of the classes in
$\mathcal{S}_2$ is $i$, and the other $k_2-1$ labels are from the set
of $R_{j}^{i,i}-1$ labels with lower score than $i$, the
classification of $x_j^{(i)}$ will be correct.  This implies that
there are ${R_{j}^{i,i}-1}\choose{k_2-1}$ such subsets $\mathcal{S}_2$
where $x_j^{(i)}$ is classified correctly, and therefore the average
test accuracy for all ${k_1}\choose{k_2}$ subsets $\mathcal{S}_2$ is
\begin{equation}\label{eq:avtestrisk}
\text{ATA}_{k_2} = \frac{1}{{{k_1}\choose{k_2}}}\frac{1}{r k_2} \sum_{i=1}^{k_1} \sum_{j=1}^{r} {{R_{j}^{i,i}-1}\choose{k_2-1}}.
\end{equation}

\subsection{Toy Example: Bivariate Normal}
\label{sec:toyExA}

Let us illustrate these ideas using a toy example.  Let $(Y, X)$ have
a bivariate normal joint distribution,
\[
(Y, X) \sim N\left(\begin{pmatrix}0 \\0\end{pmatrix}, \begin{pmatrix}1 & \rho \\ \rho & 1\end{pmatrix}\right),
\]
as illustrated in Figure \ref{fig:toy1}(a).  Therefore, for a given
randomly drawn label $Y$, the conditional
distribution of $X$ for that label is univariate normal with mean $\rho Y$ and variance $1-\rho^2$,
\[
X|Y = y \sim N(\rho Y, 1-\rho^2).
\]
Supposing we draw $k = 3$ labels $\{y_1,y_2, y_3\}$, the classification
problem will be to assign a test instance $X^*$ to the correct label.
The test instance $X^*$ would be drawn with equal probability from one
of three conditional distributions $ X | Y=y^{(i)}$, as illustrated in
Figure \ref{fig:toy1}(b, top).  The Bayes rule assigns $X^*$ to the
class with the highest density $p(x|y_i)$, as illustrated by Figure
\ref{fig:toy1}(b, bottom): it is therefore a marginal classifier, with
score function
\[
M_{y^{(i)}}(x) = \log(p(x|y^{(i)})) = -\frac{(x - \rho y)^2}{2(1-\rho^2)}  + \text{const.}
\]

\begin{figure}[p]
\centering
\begin{tabular}{cc}
\begin{myfont}Joint distribution of $(X, Y)$\end{myfont} & 
\begin{myfont}Problem instance with $k = 3$\end{myfont}\\
\multirow{3}{*}{\includegraphics[scale = 0.5, clip = true, trim = 0 0 0 0.5in]{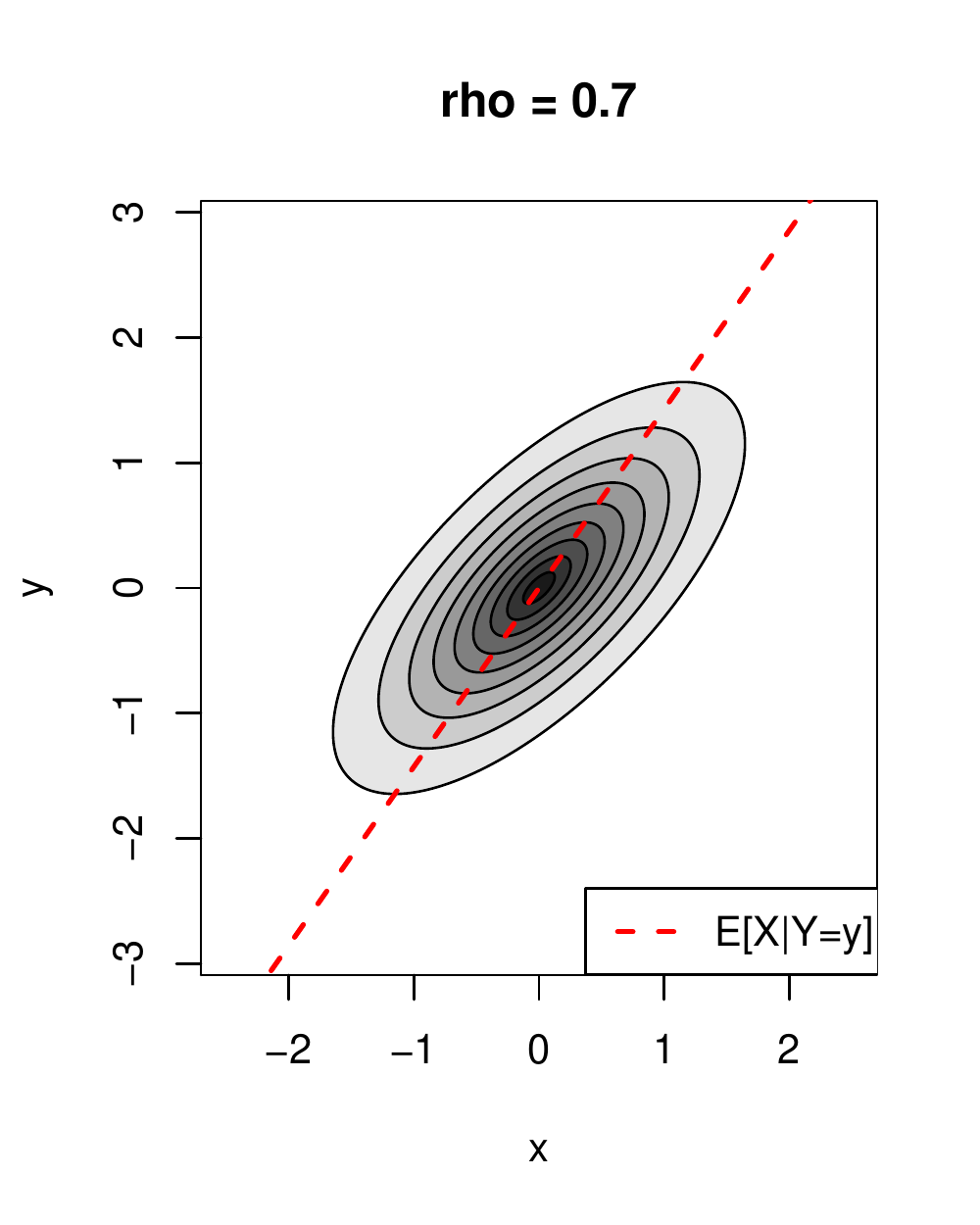}} & \\
& \includegraphics[scale = 0.5, clip = true, trim = 0 0.8in 0 0.8in]{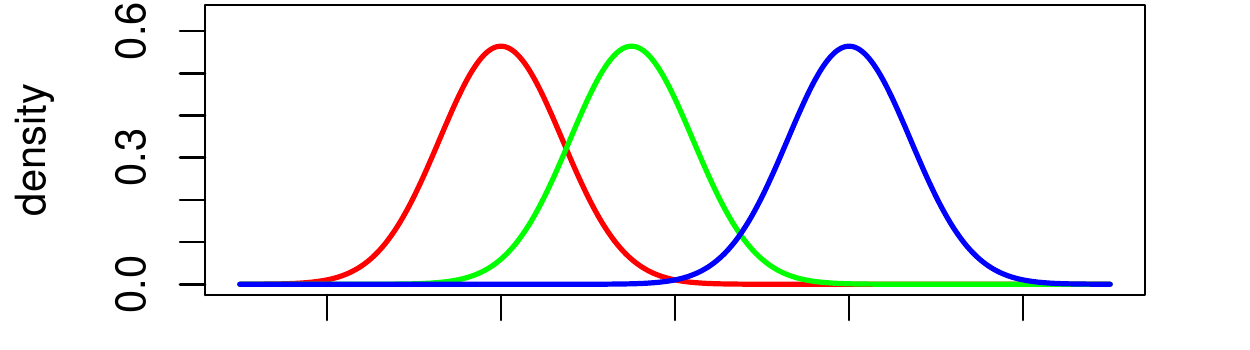}\\
 &  \includegraphics[scale = 0.5, clip = true, trim = 0 0 0 0.5in]{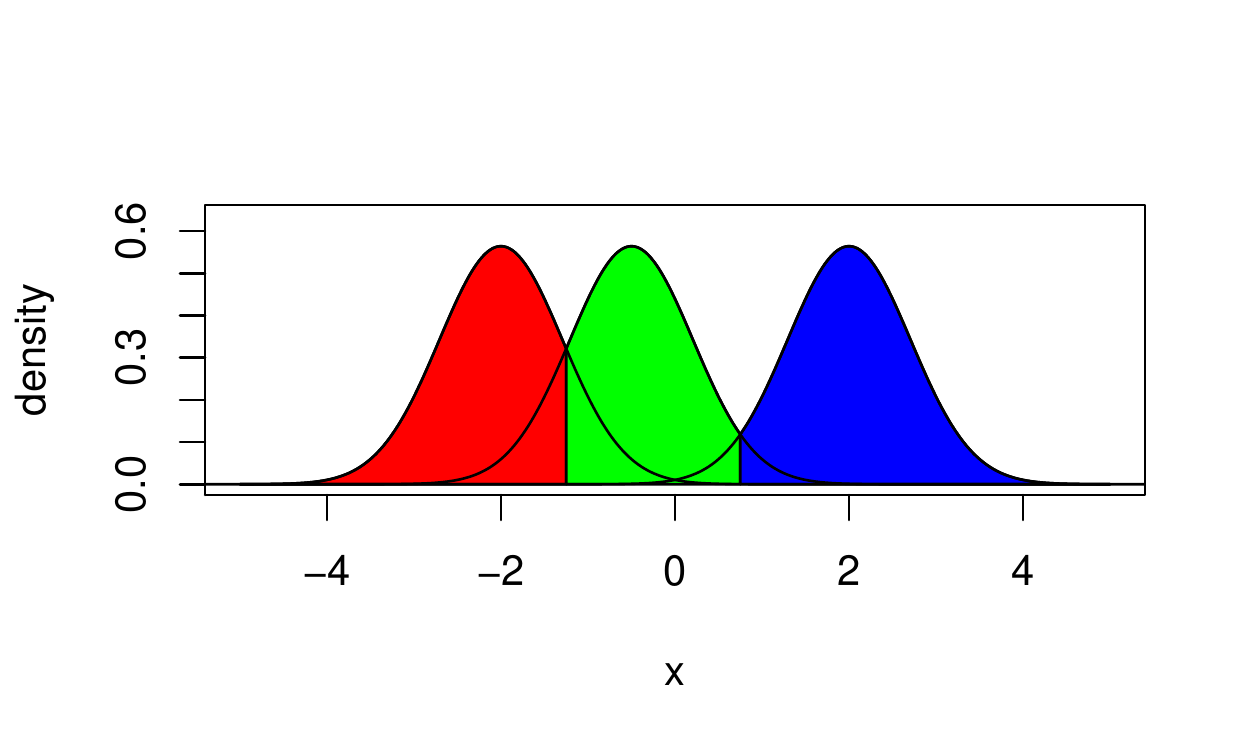}\\
(a) & (b)
\end{tabular}

\caption{\textbf{Toy example:}
\emph{Left:} The joint distribution of $(X, Y)$ is bivariate normal with correlation $\rho = 0.7$.
\emph{Right:} A typical classification problem instance from the bivariate normal model with $k = 3$ classes.
\emph{(Top):} the conditional density of $X$ given label $Y$, for $Y = \{y^{(1)}, y^{(2)}, y^{(3)}\}$.
\emph{(Bottom):} the Bayes classification regions for the three classes.}\label{fig:toy1}
\end{figure}

\begin{figure}[p]
\centering
\includegraphics[scale = 0.7, clip = true, trim = 0 0 0 0.5in]{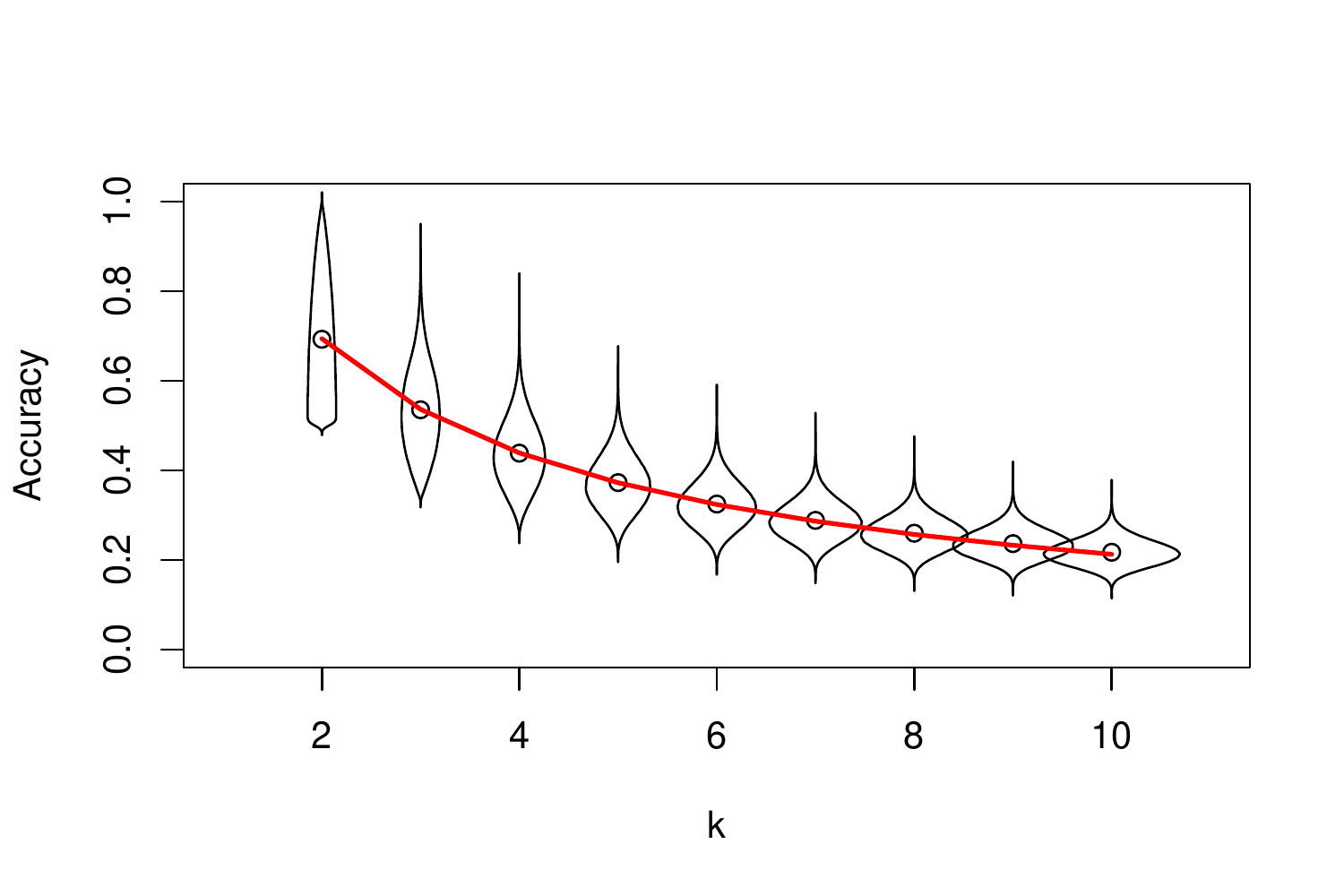}

\caption{\textbf{Generalization accuracy for toy example:} The distribution of the generalization accuracy
  for $k = 2,3,\hdots, 10$ for the bivariate normal model with $\rho =
  0.7$.  Circles indicate the average generalization accuracy $\text{AGA}_k$; the red
  curve is the theoretically computed average accuracy.}\label{fig:toy2}
\end{figure}
\section{Extrapolation}
\label{sec:extrapolation}

For this model, the generalization accuracy of the Bayes rule for any
label set $\{y^{(1)},\hdots, y^{(k)}\}$ is given by
\begin{align*}
\text{GA}_k(y_1,\hdots, y_k) &= \frac{1}{k}\sum_{i=1}^k \Pr_{X \sim p(x|y_i)}[p(X|y_i) = \max_{j=1}^k p(X|y_j)]
\\&= \frac{1}{k}\sum_{i=1}^k \Phi\left(\frac{y^{[i+1]} - y^{[i]}}{2\sqrt{1-\rho^2}}\right) - \Phi\left(\frac{y^{[i-1]} - y^{[i]}}{2\sqrt{1-\rho^2}}\right),
\end{align*}
where $\Phi$ is the standard normal cdf, $y^{[1]} < \cdots < y^{[k]}$
are the sorted labels, $y^{[0]} = -\infty$ and $y^{[k+1]} =
\infty$.  We numerically computed $\text{GA}_k(Y_1,\hdots, Y_k)$ for
randomly drawn labels $Y_1,\hdots, Y_k \stackrel{iid}{\sim} N(0, 1)$, and 
the distributions of $\text{GA}_k$ for $k = 2,\hdots, 10$ are
illustrated in Figure \ref{fig:toy2}.  The mean of the distribution of
$\text{GA}_k$ is the $k$-class average accuracy, $\text{AGA}_k$. The
theory presented in the next section deals with how to analyze the
average accuracy $\text{AGA}_k$ as a function of $k$.

The section is organized as follows.  We begin by introducing an
explicit formula for the average accuracy $\text{AGA}_{k}$.  The
formula reveals that $\text{AGA}_{k}$ is determined by moments of a
one-dimensional function ${D}(u)$.  Using this formula, we can estimate 
${D}(u)$ using subsampled accuracies.  These estimates allow us to extrapolate the average
generalization accuracy to an arbitrary number of labels.

The result of our analysis is to expose the average accuracy
$\text{AGA}_{k}$ as the weighted average of a function ${D}(u)$,
where ${D}(u)$ is independent of $k$, and where $k$ only changes
the weighting.  The result is stated as follows.

\begin{theorem}\label{theorem:avrisk_identity}
Suppose $\pi$, $\{F_y\}_{y \in \mathcal{Y}}$, and score functions $M_y$
satisfy the tie-breaking condition.  Then, there exists a cumulative
distribution function ${D}(u)$ defined on the interval $[0,1]$
such that
\begin{equation}\label{eq:avrisk_identity}
\text{AGA}_{k} = 1 - (k-1) \int {D}(u) u^{k-2} du.
\end{equation}
\end{theorem}

The tie-breaking allows us
to neglect specifying the case when
margins are tied.
\begin{definition}
\emph{Tie-breaking condition}: for all $x \in \mathcal{X}$,
$M_Y(x) \neq M_{Y'}(x)$
with probability one for $Y, Y'$ independently drawn from $\pi$.
\end{definition}
In practice, one can simply break ties randomly,
which is mathematically equivalent to adding a small amount of random
noise $\epsilon$ to the function $\mathcal{M}$.

\subsection{Analysis of Average Accuracy}

For the following discussion, we often consider a random label with
its associated score function and example vector. Explicitly, this
sampling can be written:
\[Y \sim \pi,\, M_{Y}|Y \sim \nu_{Y},\, X|Y \sim F_{Y}. \]
%\[Y^* \sim \pi,\, M_{Y^*}|Y^* \sim \nu_{Y^*},\, X^*|Y^* \sim F_{Y^*},\, \]
%\[Y' \sim \pi,\, M_{Y'}|Y' \sim \nu_{Y'}\, X'|Y' \sim F_{Y'},\, \]
Similarly we use $(Y',M_{Y'},X')$ and $(Y^*,M_{Y^*},X^*)$ for two more
triplets with independent and identical distributions. Specifically,
$X^*$ will typically note the test example, and therefore $Y^*$ the
true label and $M_{Y^*}$ its score function.

The function ${D}$ is related to a favorability
function. Favorability measures the probability that the score for the
example $x^*$ is going to be maximized by a particular score function $m_y$,
compared to a random competitor $M_{Y'}$.  Formally, we write
\begin{equation}\label{eq:U_function}
U_{x^*}(m_{y}) = \Pr[m_{y}(x^*) > M_{Y'}(x^*)].
\end{equation}

Note that for fixed example $x^*$, favorability is monotonically
increasing in $m_{y}(x^*)$.  If $m_y(x^*) > m_{y^\dagger}(x^*)$, then
$U_{x^*}(y) > U_{x^*}(y^\dagger)$, because the event $\{m_{y}(x^*) >
M_{Y'}(x^*)\}$ contains the event $\{m_{y^\dagger}(x^*) >
M_{Y'}(x^*)\}$.

Therefore, given labels $y^{(1)},\hdots,y^{(k)}$ and test instance
$x^*$, we can think of the classifier as choosing the label with the
greatest favorability:
\[
\hat{y} = \argmax_{y^{(i)} \in \mathcal{S}} m_{y^{(i)}}(x^*) = \argmax_{y^{(i)} \in \mathcal{S}} U_{x^*}(m_{y^{(i)}}).
\]
Furthermore, via a conditioning argument, we see that this is still
the case even when the test instance and labels are random:
\[
\hat{Y} = \argmax_{Y^{(i)} \in \mathcal{S}} M_{Y^{(i)}}(X^*) = \argmax_{Y^{(i)} \in \mathcal{S}} U_{X^*}(M_{Y^{(i)}}).
\]

The favorability takes values between 0 and 1, and when any of its
arguments are random, it becomes a random variable with a distribution
supported on $[0,1]$.  In particular, we consider the following two
random variables:
\begin{itemize}
\item[a.] the \emph{incorrect-label} favorability $U_{x^*}(M_Y)$
  between a given fixed test instance $x^*$, and the score function of
  a random incorrect label $M_{Y}$, and
\item[b.] the \emph{correct-label} favorability $U_{X^*}(M_{Y^*})$
  between a random test instance $X^*$, and the score function of the
  correct label, $M_{Y^*}$.
\end{itemize}
\subsubsection{Incorrect-Label Favorability}
The incorrect-label favorability can be written explicitly as
\begin{equation}
U_{x^*}(M_Y) = \Pr[M_{Y}(x^*) > M_{Y'}(x^*)|M_{Y}].
\end{equation}
Note that $M_Y$ and $M_{Y'}$ are identically distributed, and are both
are unrelated to $x^*$ that is fixed. This leads to the following
result:
\begin{lemma}\label{lemma:U_function}
Under the tie-breaking condition, the incorrect-label favorability
$U_{x^*}(M_Y)$ is uniformly distributed for any $x^* \in \mathcal{X}$,
meaning \begin{equation}\label{eq:Uniform} \Pr[U_{x^*}(M_Y) \leq u] = u
\end{equation}
for all $u \in [0,1].$
\end{lemma}

\begin{proof} Write $U_{x^*}(M_Y) = \Pr[Z > Z'|Z]$, where $Z = M_Y(x)$ and $Z' =
M_{Y'}(x)$ for $Y, Y' \stackrel{i.i.d.}{\sim} \pi$.
The tie-breaking condition implies that $\Pr[Z=Z']=0$.  Now observe that for independent random variables $Z, Z'$ with $Z \stackrel{D}{=} Z'$ and $\Pr[Z=Z']=0$, the conditional probability
$\Pr[Z > Z'|Z]$ is uniformly distributed.
\end{proof}

%\rule{0.7em}{0.7em}

\subsubsection{Correct-Label Favorability}

The correct-label favorability is 
\begin{equation}
U^* = U_{X^*}(M_{Y^*}) = \Pr[M_{Y^*}(X^*) > M_{Y'}(X^*)|Y^*,M_{Y^*},X^*].
\end{equation}
The distribution of $U^*$ will depend on $\pi$, $\{F_y\}_{y \in \mathcal{S}}$ and $\{\nu_y\}_{y \in \mathcal{S}}$, and
generally cannot be written in a closed form.  However, this
distribution is central to our analysis--indeed, we will see that the
function ${D}$ appearing in theorem \ref{theorem:avrisk_identity}
is defined as the cumulative distribution function of $U^*$.

The special case of $k=2$ shows the relation between the distribution
of $U^*$ and the average generalization accuracy, $\text{AGA}_2$. In
the two-class case, the average generalization accuracy is the
probability that a random correct label score function gives a larger
value than a random distractor:
\[
\text{AGA}_2 = \Pr[M_{Y^*}(X^*) > M_{Y'}(X^*)].
\]
where $Y^*$ is the correct label, and $Y'$ is a random incorrect
label.  If we condition on $Y^*$, $M_{Y^*}$ and $X^*$, we get
\[
\text{AGA}_2 = \E[\Pr[M_{Y^*}(X^*) > M_{Y'}(X^*)|Y^*, M_{Y^*}, X^*]].
\]
Here, the conditional probability inside the expectation is the
correct-label favorability.  Therefore,
\[
\text{AGA}_2 = \E[U^*] = \int {D}(u) du,
\]
where ${D}(u)$ is the cumulative distribution function of $U^*$,
${D}(u) = \Pr[U^* \leq u]$.  Theorem \ref{theorem:avrisk_identity}
extends this to general $k$; we now give the proof.\newline

%\noindent\textsl{Proof of Theorem \ref{theorem:avrisk_identity}}.

\begin{proof} Without loss of generality, suppose that the true label is $Y^*$ and
the incorrect labels are $Y^{(1)},\hdots, Y^{(k-1)}$.  We have
\[
\text{AGA}_k = \Pr[M_{Y^*}(X^*) > \max_{i=1}^{k-1} M_{Y^{(i)}}(X^*)]
= \Pr[U^* > \max_{i=1}^{k-1} U_{X^*}(M_{Y^{(i)}})],
\]
recalling that $U^* = U_{X^*}(M_{Y^*})$.  Now, if we condition on $X^*
= x^*$, $Y^* = y^*$ and $M_{Y^*} = m_{y^*}$, then the random variable
$U^*$ becomes fixed, with value
\[
u^* = U_{x^*}(m_{y^*}).
\]
Therefore,
\begin{align*}
\text{AGA}_k &=\E[\Pr[U^* > \max_{i=1}^{k-1} U_{X^*}(M_{Y^{(i)}})|X^* = x^*, Y^* = y^*, M_{Y^*} = m_{y^*}]]
\\&= \E[\Pr[U^* > \max_{i=1}^{k-1} U_{X^*}(M_{Y^{(i)}})|X^* = x^*, U^* = u^*]].
\end{align*}
Now define $U_{max, k-1} = \max_{i=1}^{k-1} U_{X^*}(M_{Y^{(i)}})$. 
Since by Lemma \ref{lemma:U_function},
$U_{X^*}(M_{Y^{(i)}})$ are i.i.d. uniform conditional on $X^* = x^*$, we know that
\begin{equation}\label{eq:umax_beta}
U_{max, k-1}|X^* = x^* \sim \text{Beta}(k-1, 1). 
\end{equation}
Furthermore, $U_{max, k-1}$ is independent of $U^*$ conditional on
$X^*$.  Therefore, the conditional probability can be computed as
\[
\Pr[U^* > U_{max, k-1}|X^* = x^*, U^* = u^*] = \int_{u^*}^1 (k-1) u^{k-2} du.
\]
Consequently,
\begin{align}
\text{AGA}_k &= \E[\Pr[U^* > \max_{i=1}^{k-1} U_{X^*}(M_{Y^{(i)}})|X^* = x^*, U^* = u^*]]
\\&= \E[\int_0^{U^*} (k-1) u^{k-2} du|U^* = u^*]
\\&= \E[\int_0^1 I\{u \leq U^*\} (k-1) u^{k-2} du ]
\\&= (k-1) \int_0^1 \Pr[U^* \geq u] u^{k-2} du
\\&= 1 - (k-1) \int_0^1 \Pr[U^* \leq u] u^{k-2} du. \label{eq:lala}
\end{align}
By defining ${D}(u)$ as the cumulative distribution function of
$U^*$ on $[0,1]$,
\begin{equation}\label{eq:Kbar}
{D}(u) = \Pr[U_{X^*}(M_{Y^*}) \leq u],
\end{equation}
and substituting this definition into \eqref{eq:lala}, we obtain the identity \eqref{eq:avrisk_identity}.
\end{proof}

Theorem \ref{theorem:avrisk_identity} expresses the average accuracy
as a weighted integral of the function ${D}(u)$.  Essentially, this theoretical result allows us
to reduce the problem of estimating $\text{AGA}_k$ to one of estimating $D(u)$.
But how shall we estimate $D(u)$ from data?
We propose using non-parametric regression for this purpose in Section \ref{sec:extrapolation_estimation}.

\subsection{Favorability and Average Accuracy for the Toy Example}

Recall that for the toy example from Section \ref{sec:toyExA}, the
score function $M_{y}$ was a non-random function of $y$ that measures
the distance between $x$ and $\rho y$
\[
M_{y}(x^*) = \log(p(x^*|y)) = -\frac{(x^* - \rho y)^2}{2(1-\rho^2)} .
\]

For this model, the favorability function $U_{x^*}(m_y)$ compares the
distance between $x^*$ and $\rho y$ to the distance between $x^*$ and
$\rho Y'$ for a randomly chosen distractor $ Y'\sim N(0,1)$:
\begin{align*}
U_{x^*}(m_y) &= \Pr[|\rho y - x^*|> |\rho Y' - x^*|] 
\\&= \Phi\left(\frac{x^* + |\rho y - x^*|}{\rho}\right) - \Phi\left(\frac{x^* - |\rho y - x^*|}{\rho}\right),
\end{align*}
where $\Phi$ is the standard normal cumulative distribution function.
Figure \ref{fig:toy3}(a) illustrates the level sets of the function
$U_{x^*}(m_y)$.  The highest values of $U_{x^*}(m_y)$ are near the
line $x^* = \rho y$ corresponding to the conditional mean of $X|Y$, and as
one moves farther from the line, $U_{x^*}(m_y)$ decays.  Note, however,
that large values of $x^*$ and $y$ (with the same sign) result in
larger values of $U_{x^*}(m_y)$ since it becomes unlikely for $Y' \sim
N(0,1)$ to exceed $Y = y$.

Using the formula above, we can calculate the correct-label
favorability $U^* = U_{X^*}(M_{Y^*})$ and its cumulative distribution
function ${D}(u)$.  The function ${D}$ is illustrated in Figure
\ref{fig:toy3}(b) for the current example with $\rho = 0.7$.  The red
curve in Figure \ref{fig:toy2} was computed using the formula
\[
\text{AGA}_k = 1-(k-1) \int {D}(u) u^{k-2} du.
\]

It is illuminating to consider how the average accuracy curves and the
${D}(u)$ functions vary as we change the parameter $\rho$.  Higher
correlations $\rho$ lead to higher accuracy, as seen in Figure
\ref{fig:toy4}(a), where the accuracy curves are shifted upward as
$\rho$ increases from 0.3 to 0.9.  The favorability $U_{x^*}(m_y)$
tends to be higher on average as well, which leads to lower values of
the cumulative distribution function--as we see in Figure
\ref{fig:toy4}(b), where the function ${D}(u)$ becomes smaller as
$\rho$ increases.

\begin{figure}[p]
\centering
\begin{tabular}{cc}
\begin{myfont}$U_x^*(M_y)$ for $\rho=0.7$\end{myfont}
& 
\begin{myfont}$D(u)$ for $\rho = 0.7$\end{myfont}\\
\includegraphics[scale = 0.6, clip = true, trim = 0.1in 0 0 0.8in]{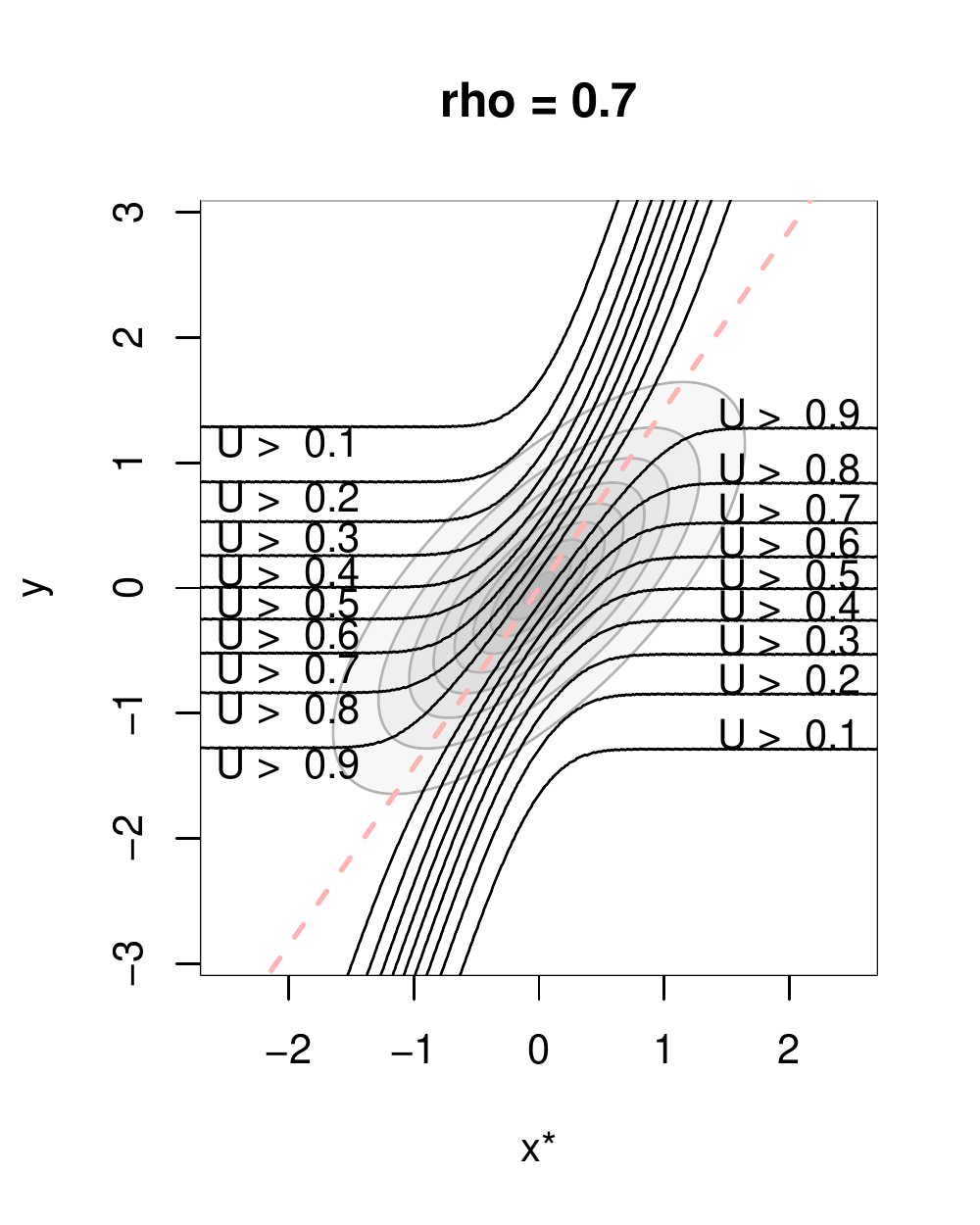} &
\includegraphics[scale = 0.6, clip = true, trim = 0.22in -0.3in 0 0.5in]{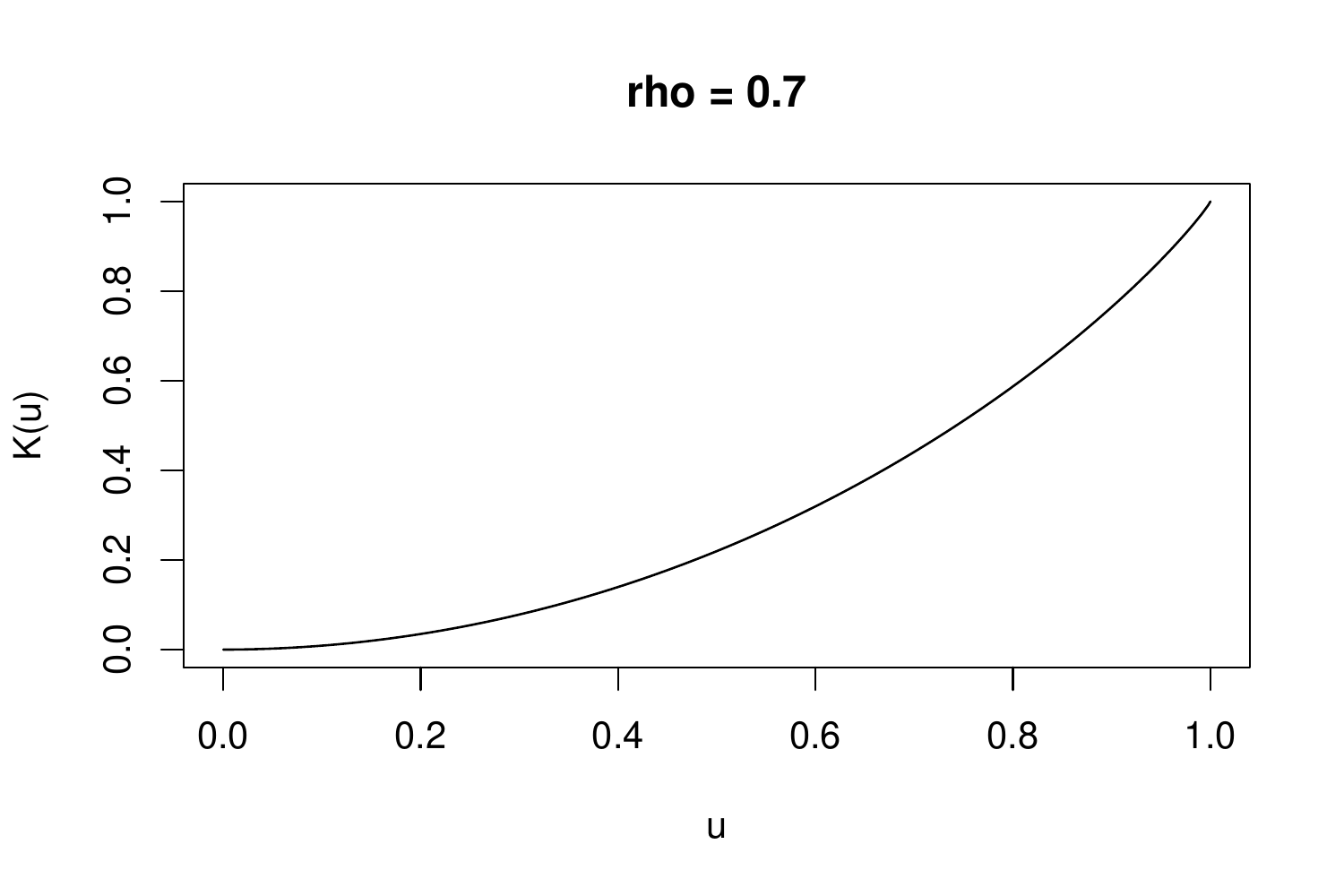}
\end{tabular}

\caption{\textbf{Favorability for toy example:}
\emph{Left:} The level curves of the function $U_{x^*}(M_y)$ in the bivariate normal model with $\rho = 0.7$.
\emph{Right:} The function ${D}(u)$ gives the cumulative distribution function of the random variable $U_{X^*}(M_Y)$.}\label{fig:toy3}
\end{figure}

\begin{figure}[p]
\centering
\begin{tabular}{cc}
\begin{myfont}\hspace{0.2in}$D(u)$\end{myfont} &
\begin{myfont}\hspace{0.4in}Average Accuracy\end{myfont} \\
\includegraphics[scale = 0.6, clip = true, trim = 0.22in 0 0.2in 0.6in]{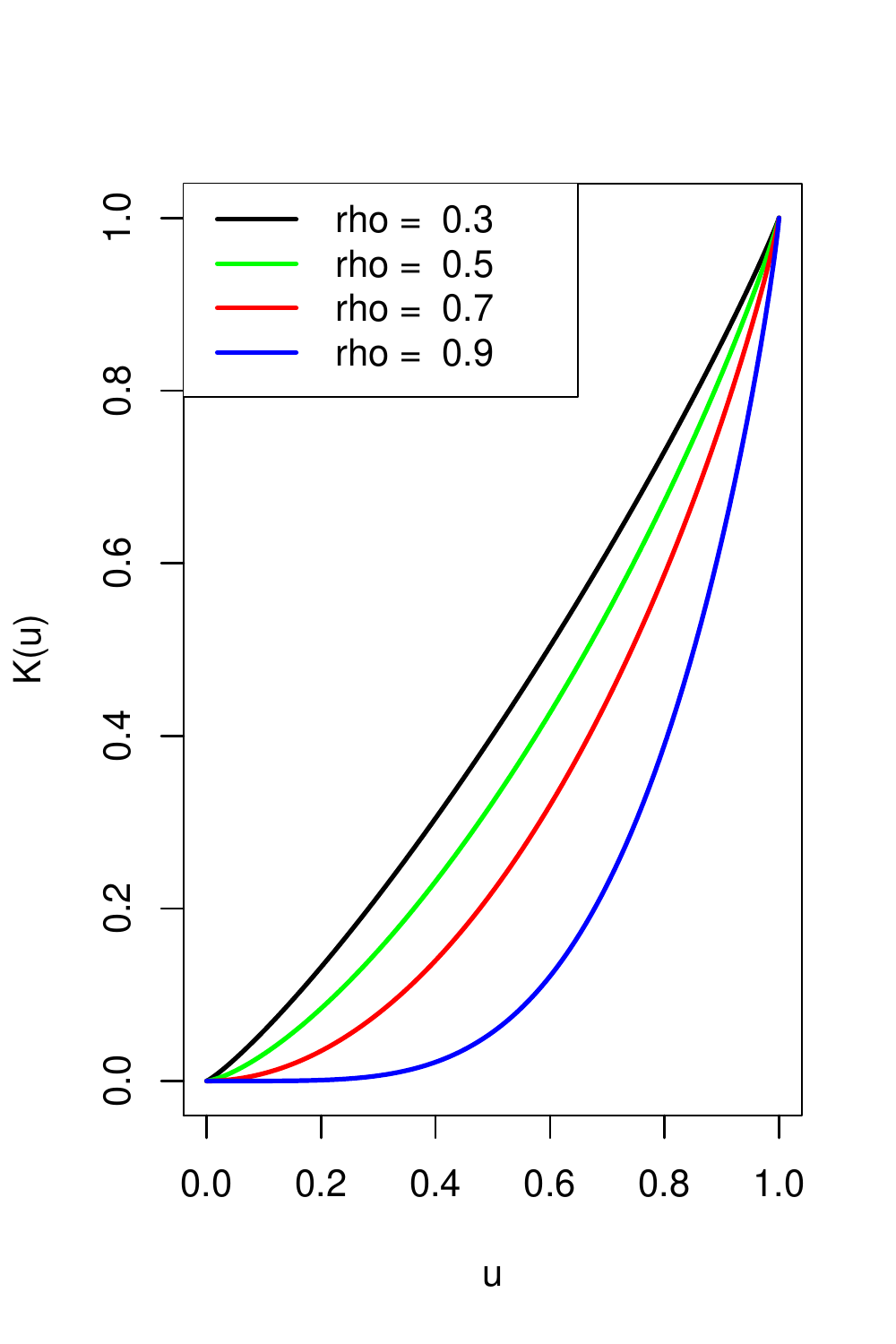} &
\includegraphics[scale = 0.6, clip = true, trim = 0.0in 0 0.2in 0.6in]{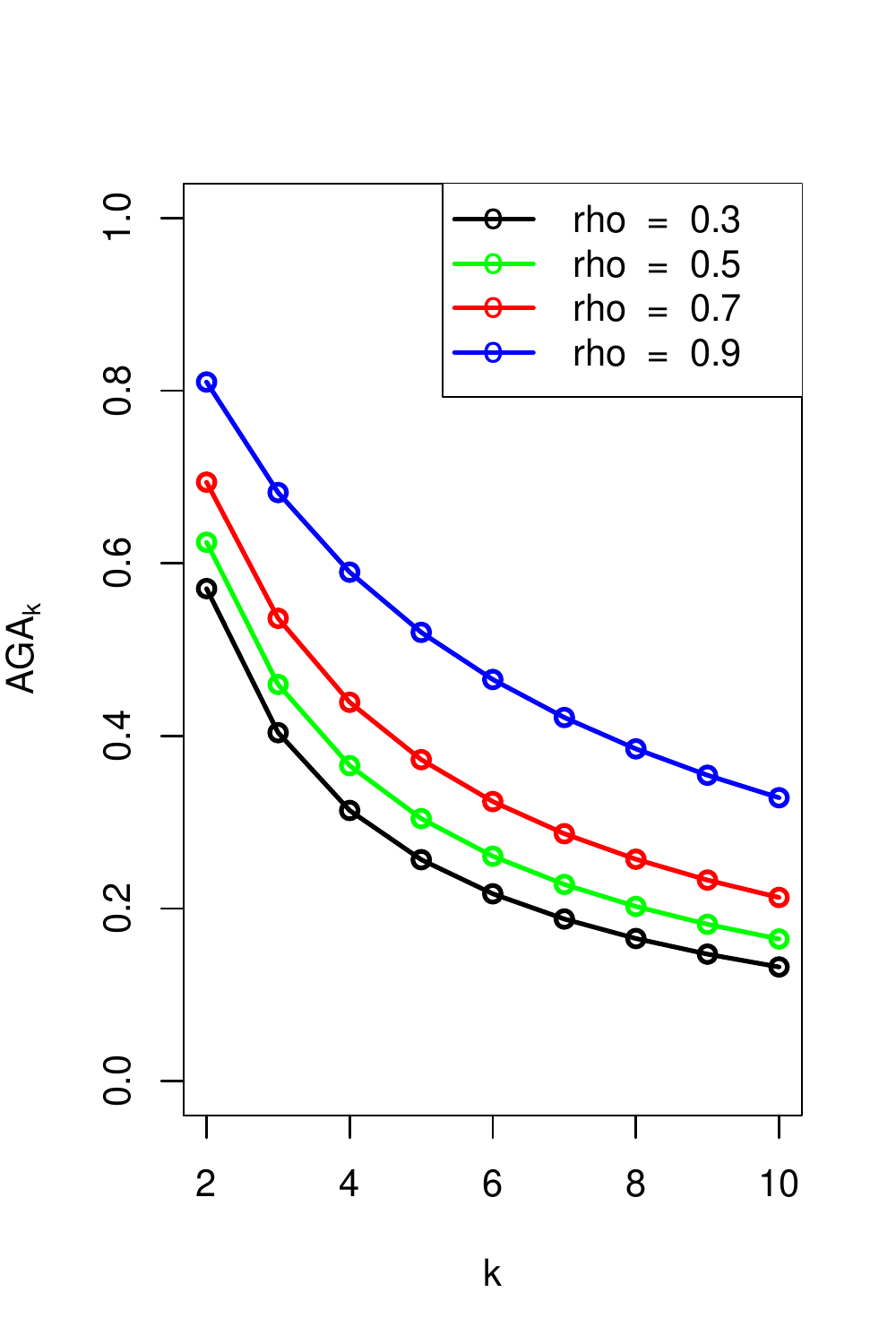}
\end{tabular}

\caption{\textbf{Average accuracy with different $\rho$'s:}
\emph{Left:} The average accuracy $\text{AGA}_k$. \emph{Right:} ${D}(u)$ function for the bivariate normal model with $\rho \in \{0.3, 0.5, 0.7, 0.9\}$.
}\label{fig:toy4}
\end{figure}

\subsection{Estimation}\label{sec:extrapolation_estimation}

Next, we discuss how to use data from smaller classification tasks to
extrapolate average accuracy.  Assume that we have data from a
$k_1$-class random classification task, and would like to estimate the
average accuracy $\text{AGA}_{k_2}$ for $k_2>k_1$ classes.  
Our estimation method will use the $k$-class average test accuracies,
$\text{ATA}_2,...,\text{ATA}_{k_1}$ (see Eq \ref{eq:avtestrisk}), for
its inputs.

The key to understanding the behavior of the average accuracy
$\text{AGA}_k$ is the function ${D}$.  We adopt a linear model
\begin{equation}\label{eq:linearKu}
{D}(u) = \sum_{\ell = 1}^m \beta_\ell h_\ell(u),
\end{equation}
where $h_\ell(u)$ are known basis functions, and $\beta_\ell$ are the
linear coefficients to be estimated.  Since our proposed method is
based on the linearity assumption \eqref{eq:linearKu}, we refer to it as ClassExReg,
meaning \emph{Classification Extrapolation using Regression}.

Conveniently, $\text{AGA}_k$ can also be expressed in terms of the $\beta_\ell$ coefficients.
If we plug in the assumed linear model \eqref{eq:linearKu} into the
identity \eqref{eq:avrisk_identity}, then we get
\begin{align}
1 - \text{AGA}_k &= (k-1)\int {D}(u) u^{k-2} du
\\&= (k-1)\int_0^1 \sum_{\ell = 1}^m \beta_\ell h_\ell(u) u^{k-2} du
\\&= \sum_{\ell = 1}^m \beta_\ell H_{\ell,k}, \label{eq:avrisk_linear}
\end{align}
where
\begin{equation}
H_{\ell,k} = (k-1) \int_0^1 h_\ell(u) u^{k-2} du.
\end{equation}
The constants $H_{\ell, k}$ are moments of the basis function
$h_\ell$.  Note that $H_{\ell, k}$ can be precomputed numerically for any $k \geq 2$.

Now, since the test accuracies $\text{ATA}_k$ are unbiased estimates
of $\text{AGA}_{k}$, this implies that the regression estimate
\[
\hat{\beta} = \argmin_\beta \sum_{k=2}^{k_1} \left( (1 - \text{ATA}_k) - \sum_{\ell=1}^m \beta_\ell H_{\ell, k}\right)^2,
\]
is unbiased for $\beta$. The estimate of $\text{AGA}_{k_2}$ is similarly obtained
from \eqref{eq:avrisk_linear}, via
\begin{equation}\label{eq:avrisk_hat}
\widehat{\text{AGA}_{k_2}} = 1 - \sum_{\ell=1}^m \hat{\beta}_\ell H_{\ell, k_2}.
\end{equation}

\subsection{Model Selection}\label{sec:modelselection}

Accurate extrapolation using ClassExReg depends on a good
fit between the linear model \eqref{eq:linearKu} and the true
discriminability function $D(u)$.  However, since the function $D(u)$
depends on the unknown joint distribution of the data, it makes sense
to let the data help us choose a good basis $\{h_u\}$ from a set of
candidate bases.

Let $B_1,\hdots, B_s$ be a set of candidate bases, with $B_i = \{h_u^{(i)}\}_{u=1}^{m_i}$.  Ideally, we would like our model selection procedure to choose the $B_i$ that obtains the best root-mean-squared error (RMSE) on the extrapolation from $k_1$ to $k_2$ classes.  As an approximation, we estimate the RMSE of extrapolation from $\frac{k_1}{2}$ source classes to $k_1$ target classes, by means of the ``bootstrap principle.''  This amounts to a resampling-based model selection approach, where we perform extrapolations from $k_0 = \lfloor \frac{k_1}{2} \rfloor$ classes to $k_1$ classes, and evaluate methods based on how closely the predicted $\widehat{\text{ABA}}_{k_1}$ matches the test accuracy $\text{ATA}_{k_1}$.  To elaborate, our model selection procedure is as follows.

\begin{enumerate}
\item For $\ell=1,\hdots,L$ resampling steps:
\begin{enumerate}
\item Subsample $\mathcal{S}_{k_0}^{(\ell)}$ from $\mathcal{S}_{k_1}$ uniformly with replacement.
\item Compute average test accuracies $\text{ATA}_{2}^{(\ell)},\hdots,\text{ATA}_{k_0}^{(\ell)}$ from the subsample $\mathcal{S}_{k_0}^{(\ell)}$.
\item For each candidate basis $B_i$, with $i = 1,\hdots, s$:
\begin{enumerate}
\item Compute $\hat{\beta}^{(i,\ell)}$ by solving the least-squares problem
\[\hat{\beta}^{(i,\ell)} = \argmin_\beta \sum_{k=2}^{k_0} \left( (1 - \text{ATA}_k^{(\ell)}) - \sum_{j=1}^{m_i} \beta_j H_{j, k}^{(i)}\right)^2.\]
\item Estimate $\widehat{\text{AGA}}_{k_1}^{(i,\ell)}$ by
\[
\widehat{\text{AGA}}_{k_1}^{(i,\ell)} = \sum_{\ell=1}^{m_i} \hat{\beta}_j^{(i,\ell)} H_{j, k_1}^{(i)}.
\]
\end{enumerate}
\end{enumerate}
\item Select the basis $B_{i^*}$ by
\[
i^* = \text{argmin}_{i=1}^s \sum_{\ell=1}^L (\widehat{\text{AGA}}_{k_1}^{(i,\ell)} - \text{ATA}_{k_1})^2.
\]
\item Use the basis $B_{i^*}$ to extrapolate from $k_1$ classes (the full data) to $k_2$ classes.
\end{enumerate}

\section{Simulation Study}\label{sec:simulation_study}

We ran simulations to check how the proposed extrapolation method, ClassExReg,
performs in different settings.  The results are displayed in Figure
\ref{fig:sim_study}. 
We varied the number of classes $k_1$ in the
source data set, the difficulty of classification, and the basis
functions. We generated data according to a mixture of isotropic
multivariate Gaussian distributions: labels $Y$ were sampled from $Y
\sim N(0, I_{10})$, and the examples for each label sampled from $X|Y
\sim N(Y, \sigma^2 I_{10})$. The noise-level parameter $\sigma$
determines the difficulty of classification. Similarly to the
real-data example, we consider a 1-nearest neighbor classifier, which
is given a single training instance per class.
%We vary $k_1$,
%the number of classes in the source task; and $\sigma$, which controls
%the noise level, and hence the true average accuracy.

For the estimation, we use the model selection procedure described
in section \ref{sec:modelselection} to select the parameter $h$ of the
``radial basis''
\[
h_\ell(u) = \Phi\left(\frac{\Phi^{-1}(u) - t_\ell}{h}\right).
\]
where $t_\ell$ are a set of regularly spaced knots which are
determined by $h$ and the problem parameters. Additionally, we add a
constant element to the basis, equivalent to adding an intercept to
the linear model \eqref{eq:linearKu}. 

The rationale behind the radial basis is to model the density of
$\Phi^{-1}(U^*)$ as a mixture of gaussian kernels with variance $h^2$.
To control overfitting, the knots are separated by at least a distance
of $h/2$, and the largest knots have absolute value $\Phi^{-1}(1 -
\frac{1}{rk_1^2}).$ The size of the maximum knot is set this way since
$rk_1^2$ is the number of ranks that are calculated and used by our
method.  Therefore, we do not expect the training data to contain
enough information to allow our method to distinguish between more
than $rk_1^2$ possible accuracies, and hence we set the maximum knot
to prevent the inclusion of a basis element that has on average a
higher mean value than $u = 1-\frac{1}{rk_1^2}$.  However, in
simulations we find that the performance of the basis depends only
weakly on the exact positioning and maximum size of the knots, as long
as sufficiently large knots are included. As is the case throughout non-parametric statistics, the bandwidth $h$ is the
most crucial parameter.  In the simulation, we use a grid $h = \{0.1,
0.2, \hdots, 1\}$ for bandwidth selection.

\subsection{Comparison to Kay}
\label{sec:KDEcomparison}
In their paper,\footnote{The KDE extrapolation method is described in page 29 of supplement to \cite{Kay2008a}.  While the method is only described for
a one-nearest neighbor classifier and for the setting where there is at most
one test observation per class, we have taken the liberty of extending it to a generic multi-class classification problem.}
\cite{Kay2008a} proposed a method for extrapolating classification
accuracy to a larger number of classes. The method depends on repeated
kernel-density estimation (KDE) steps. Because the method is only
briefly motivated in the original text, we present it in our
notation. 

For $k_1$ observed classes, let $m_{y^{(j)}}(x^{(i)}_\ell)\, 1\leq i,j\leq k_1$ be
the observed score comparing feature vector of the $\ell$'th test example of the $i$'th class $x^{(i)}_\ell$
to the model trained for the $j$'th class $m_{y^{(j)}}$. 
For each feature-vector $x^{(i)}_\ell$, the density of wrong-class scores is estimated by smoothing the observed scores
with a kernel function $K(\cdot,\cdot)$ with bandwidth $h$,
\[\hat{f}_\ell^{(i)}(m) = \frac{1}{k_1-1} \sum_{j\neq i} K_h(m_{y^{(j)}}(x^{(i)}_\ell), m).\]
An estimate of favorability for $x^{(i)}_\ell$ against a
single competitor is obtained by integrating the density below the observed true score $m_{y^{(i)}}(x^{(i)}_\ell)$, 
\[acc_2(x_\ell^{(i)}) = \int_{0}^{m_{y^{(i)}}(x^{(i)}_\ell)} \hat{f}_\ell^{(i)}(m)dm,\]
and the probability of accurate classification against $K-1$ random
competitors is $acc_{K}(x^{(i)}_\ell) = (acc_2(x^{(i)}_\ell))^{K-1}$. The average of these probabilities is the estimate for generalization accuracy
\[\hat{AGA}_{K}^{(KDE)} = \frac{1}{k_1 r} \sum_{i=1}^{k_1} \sum_{\ell=1}^{r} acc_{K}(x^{(i)}_\ell).\]

Note that the KDE method depends non-trivially on the smoothing bandwidth used in the density estimation step: when the kernel bandwidth is too small compared to the number of classes, the extrapolated accuracy will be upwardly biased.  To see
this, consider a feature vector $x^{(i)}_\ell$ that is correctly classified in the original class set. That is, $m_{y^{(i)}}(x^{(i)}_\ell)>m_{y^{(j)}}(x^{(i)}_\ell)$ for every
$j\neq i$. If we use for $\hat{f}_i$ the unsmoothed empirical density,
$acc_2(x^{(i)}_\ell)=1$ and therefore $acc_K(x^{(i)}_\ell) = 1$ as well. The method relies on smoothing of each class to generate a the tail density that exceeds $m_{y^{(i)}}(x^{(i)}_\ell)$, and therefore it is highly dependent on the choice of kernel bandwidth. 

\cite{Kay2008a} recommended using a Gaussian
kernel, with the bandwidth chosen via pseudolikelihood cross-validation \citep{cao1994comparative}.
In our simulation, we tested our own implementation of the KDE method,
using the two methods for cross-validated KDE estimation provided in
the {\tt stats} package in the {\tt R} statistical computing
environment: biased cross-validation and unbiased cross-validation \citep{Scott1992}.

\subsection{Simulation Results}

We see in Figure \ref{fig:sim_study} that ClassExReg and the KDE methods with unbiased and biased
cross-validation (KDE-UCV, KDE-BCV) perform comparably in the Gaussian
simulations.  We studied how the difficulty of extrapolation relates to both the
absolute size of the number of classes and the extrapolation factor $\frac{k_2}{k_1}$.  Our simulation has two settings for $k_1 = \{500,5000\}$, and
within each setting we have extrapolations to 2 times, 4 times, 10
times, and 20 times the number of classes.

\begin{table}
\centering
\begin{tabular}{cc||c|c|c}
\hline
$k_1$ & $k_2$ & ClassExReg & KDE-BCV & KDE-UCV \\\hline 
500 & 1000 & \textbf{0.032} (0.001) & 0.090 (0.001) & 0.067 (0.001) \\
500 & 2000 & \textbf{0.044} (0.002) & 0.088 (0.001) & 0.059 (0.001) \\
500 & 5000 & 0.073 (0.004) & 0.079 (0.001) & \textbf{0.051} (0.001) \\
500 &10000 & 0.098 (0.004) & 0.076 (0.001) & \textbf{0.045} (0.001) \\\hline
5000 & 10000 & \textbf{0.009} (0.000) & 0.038 (0.000) & 0.028 (0.000) \\
5000 & 20000 & \textbf{0.015} (0.001) & 0.028 (0.000) & 0.019 (0.000) \\
5000 & 50000 & \textbf{0.032} (0.002) & 0.035 (0.000) & 0.053 (0.000) \\
5000 &100000 & \textbf{0.054} (0.003) & 0.065 (0.000) & 0.086 (0.000) \\\hline
\end{tabular}
\caption{Maximum RMSE (se) across all signal-to-noise-levels in
  predicting $\text{TA}_{k_2}$ from $k_1$ classes in multivariate
  gaussian simulation.  Standard errors were computed by nesting the
  maximum operation within the bootstrap, to properly account for the variance of a maximum of estimated means.}\label{tab:sim_max_error}
\end{table}

Within each problem setting defined by the number of
source and target classes $(k_1,k_2)$, we use the maximum RMSE across
all signal-to-noise settings to quantify the overall performance of
the method, as displayed in Table \ref{tab:sim_max_error}.

The results also indicate that more accurate extrapolation
appears to be possible for smaller extrapolation ratios
$\frac{k_2}{k_1}$ and larger $k_1$.
ClassExReg improves in worst-case RMSE when moving from $k_1 =500$ to $k_1 = 5000$ while keeping the extrapolation factor fixed, most dramatically in the case $\frac{k_2}{k_1} = 2$ when it improves from a
maximum RMSE of $0.032\pm0.001$ ($k_1 = 500$) to $0.009 \pm 0.000$ ($k_1 = 5000$), which is 3.5-fold reduction in worst-case RMSE, but
also benefiting from at least a 1.8-fold reduction in RMSE when going from
the smaller problem to the larger problem in the other three cases.

The kernel-density method produces comparable results, but is seen to depend strongly on the choice of bandwidth selection: KDE-UCV and KDE-BCV show very different performance profiles, although they differ only in the method used to choose the bandwidth.
Also, the KDE methods show significant estimation bias, as can be seen from Figure \ref{fig:sim_study_bias}.  
This can be explained by the fact that
the KDE method ignores the bias introduced by exponentiation.
That is, even if $\hat{p}$ is an unbiased estimator of $p$, $\hat{p}^k$ may \emph{not} be a very good estimate of $p^k$, since
\[
p^k = \E[\hat{p}]^k \neq \E[\hat{p}^k].
\]
unless $\hat{p}$ is a degenerate random variable
(constant). Otherwise, for large $k$, $\hat{p}^k$ may be an extremely
biased estimate of $p$.  ClassExReg avoids this source of
bias by estimating the $(k-1)$st moment of $D(u)$ directly.  As we see
in Figure \ref{fig:sim_study_bias}, correcting for the bias of
exponentiation helps greatly to reduce the overall bias.   Indeed, while
ClassExReg shows comparable bias for the 500 to 10000 extrapolation,
the bias is very well-controlled in all of the $k_1 = 5000$ extrapolations.

%\frac{1}{k_1-1} \sum_{j\neq i} K(m_{y_i}(x_j), m).\] The
%extrapolation method relies on estimating for each class in the
%source set its own incorrect-label favorability. That is, for each
%class in the source set $x_1,...,x_{k_1}$, they
%estimated \[\hat{h}(x_i) = \hat{P}(x_i > M_{y_i}(X^*))\] they
%estimated the misclassification probability against a single (random)
%competitor (what we call the incorrect-label favorability). by using
%a kernel-density estimator for the competitor distribution.  In our
%notation, they estimated \[\hat{U}_pooling and smoothikernal-ng the
%observed scores for all competitors.  \begin{enumerate}
%$P(h^{(2)}(X_i)\neq y_i)$. They then extrapolated using the
%formula \[GA_{k_2} = 1-\frac{1}{k_1}P(h^{(2)}(x_i)\neq
%y_i)^{k_2-1}.\] Their method can be interpreted within the framework
%of our theory, and often achieves good extrapolation
%performance. However, as we discuss below, estimating the probability
%$P(h(2)(X_i)\neq y_i)$ is strongly dependent on tuning parameters,
%and the approach can be unstable when the source set ($S_{k_1}$) is
%not large enough.

\begin{figure}[p]
\centering
\begin{tabular}{cc}
\multicolumn{2}{c}{\begin{myfont}Extrapolating from $k_1 = 500$\end{myfont}}\\
\begin{myfont}Predicting $\text{AGA}_{1000}$\end{myfont} &
\begin{myfont}Predicting $\text{AGA}_{2000}$\end{myfont}\\
\includegraphics[scale = 0.5, clip = true, trim = 0 0 1.25in 0.45in]{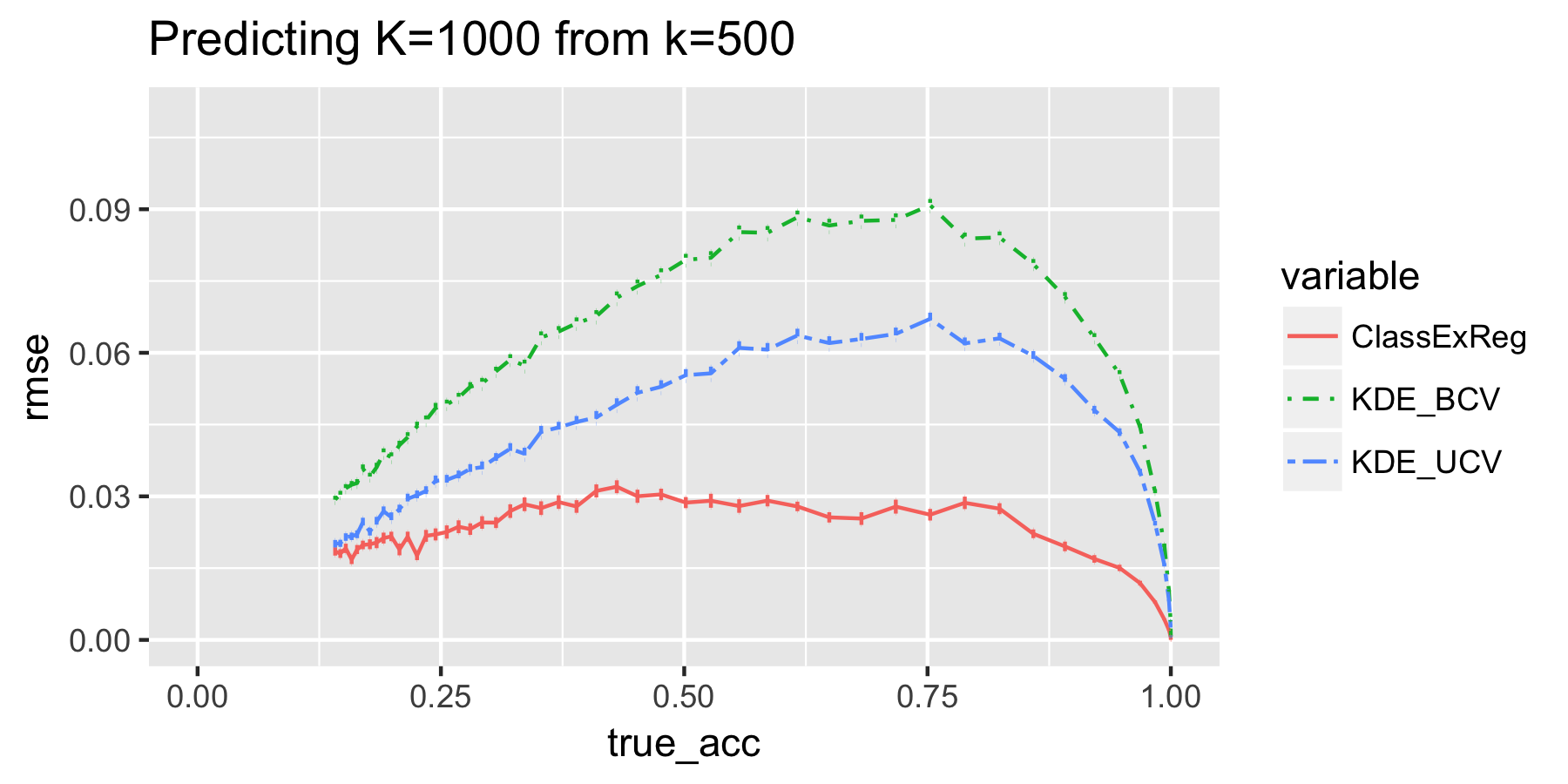} &
\includegraphics[scale = 0.5, clip = true, trim = 0 0 0 0.45in]{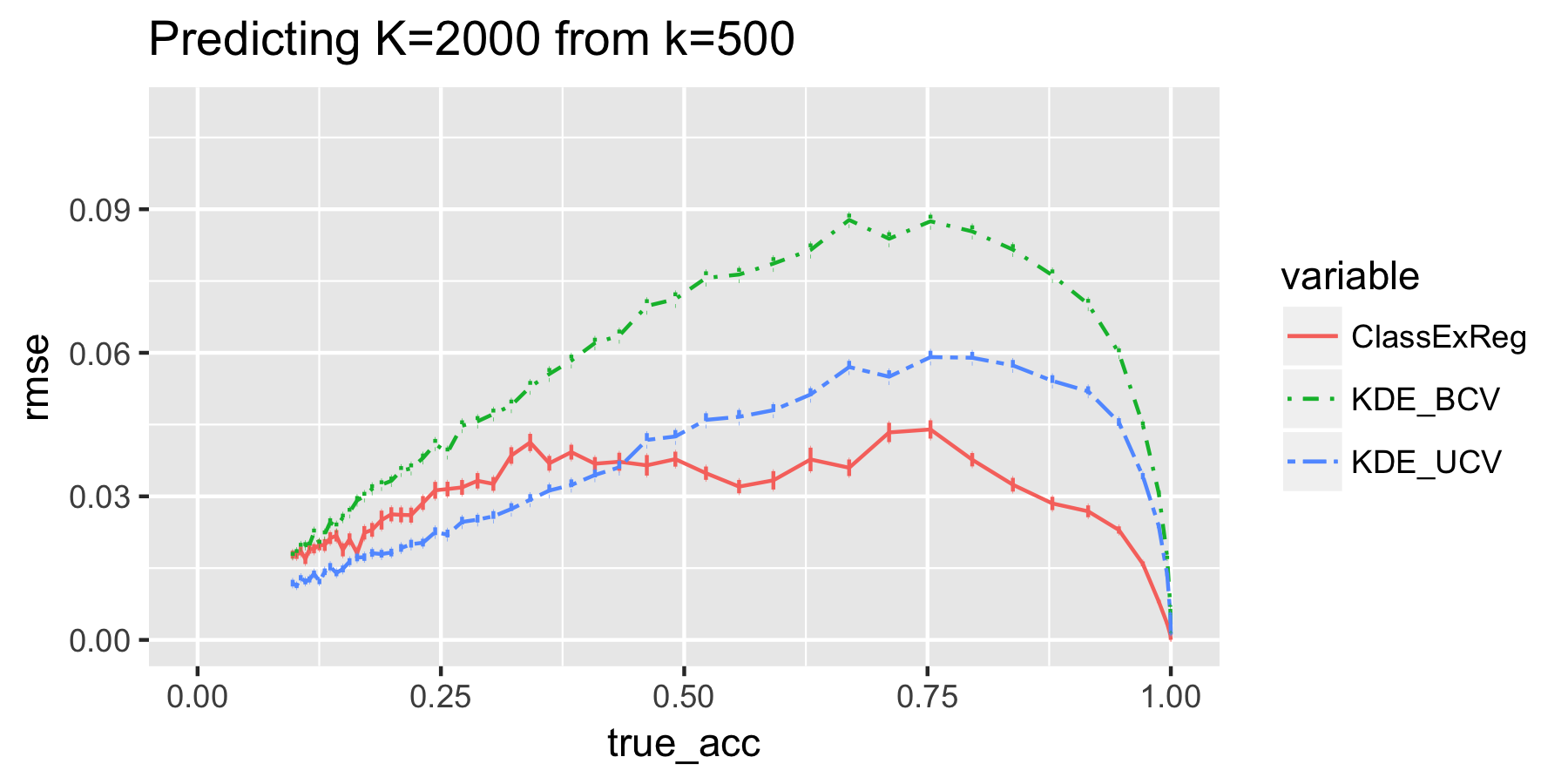}\\
\begin{myfont}Predicting $\text{AGA}_{5000}$\end{myfont} &
\begin{myfont}Predicting $\text{AGA}_{10000}$\end{myfont}\\
\includegraphics[scale = 0.5, clip = true, trim = 0 0 1.25in 0.45in]{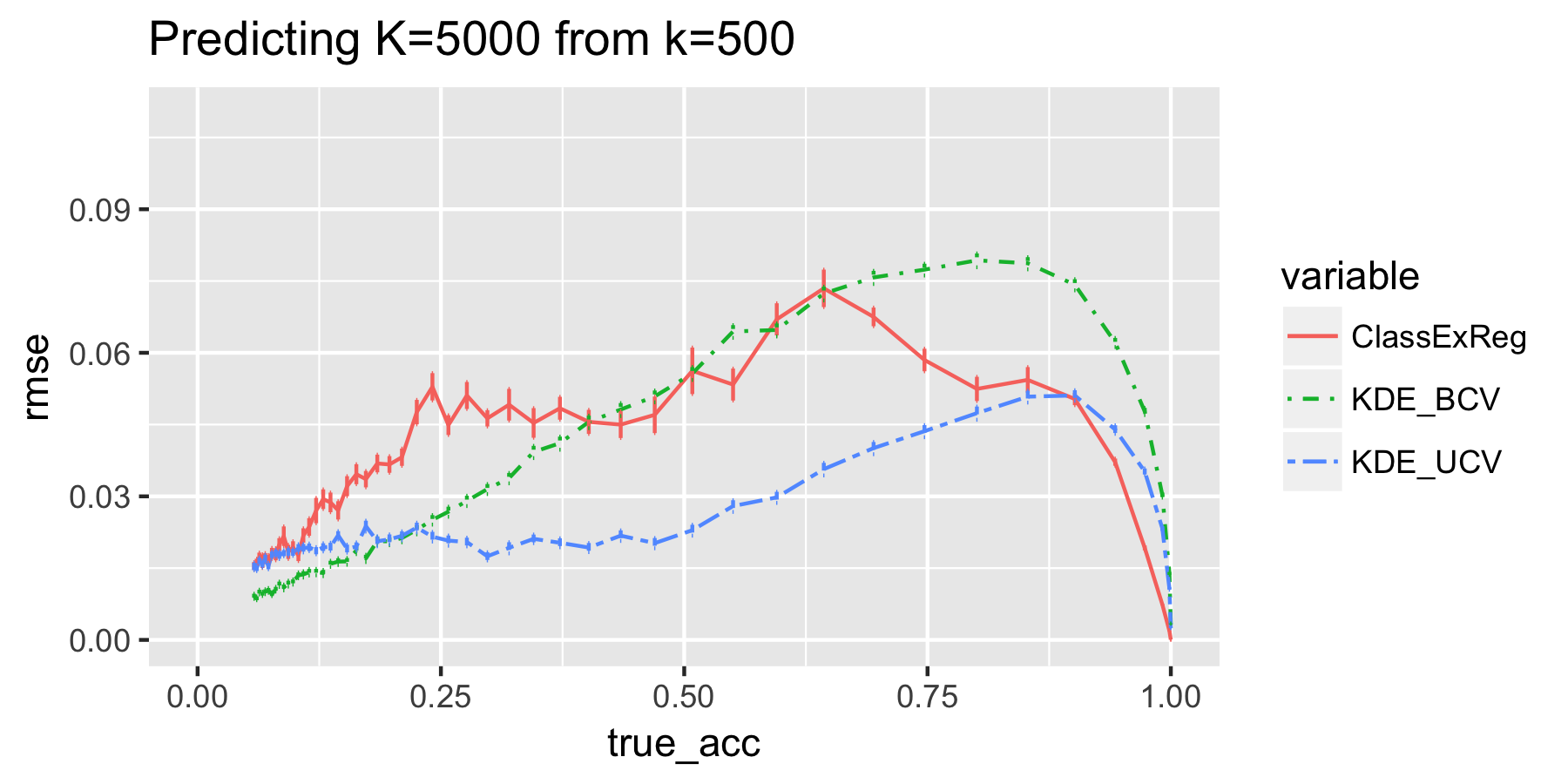} &
\includegraphics[scale = 0.5, clip = true, trim = 0 0 0 0.45in]{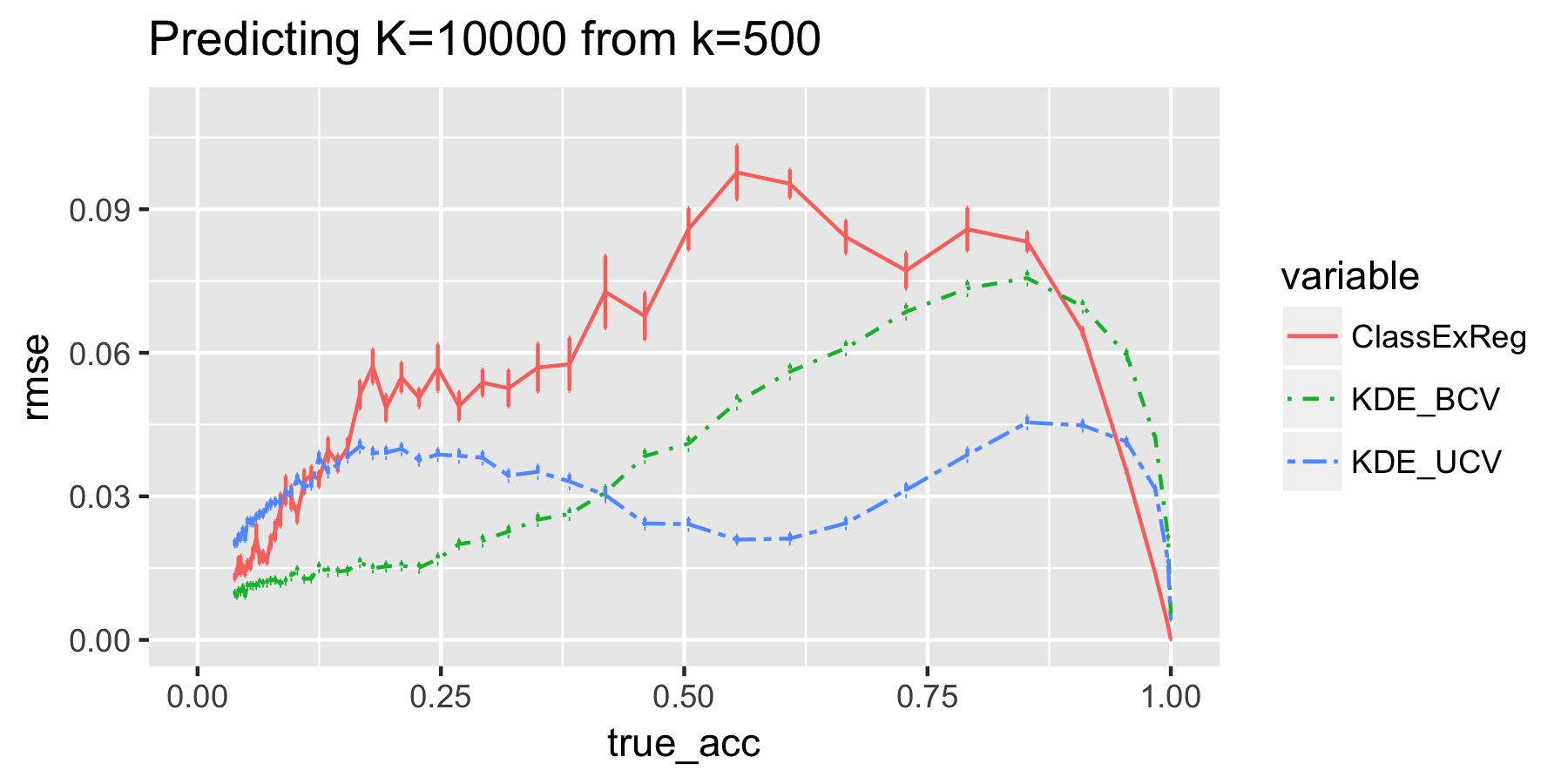}\\ 
\multicolumn{2}{c}{\begin{myfont}Extrapolating from $k_1 = 5000$\end{myfont}}\\
\begin{myfont}Predicting $\text{AGA}_{10000}$\end{myfont} &
\begin{myfont}Predicting $\text{AGA}_{20000}$\end{myfont}\\
\includegraphics[scale = 0.5, clip = true, trim = 0 0 1.25in 0.45in]{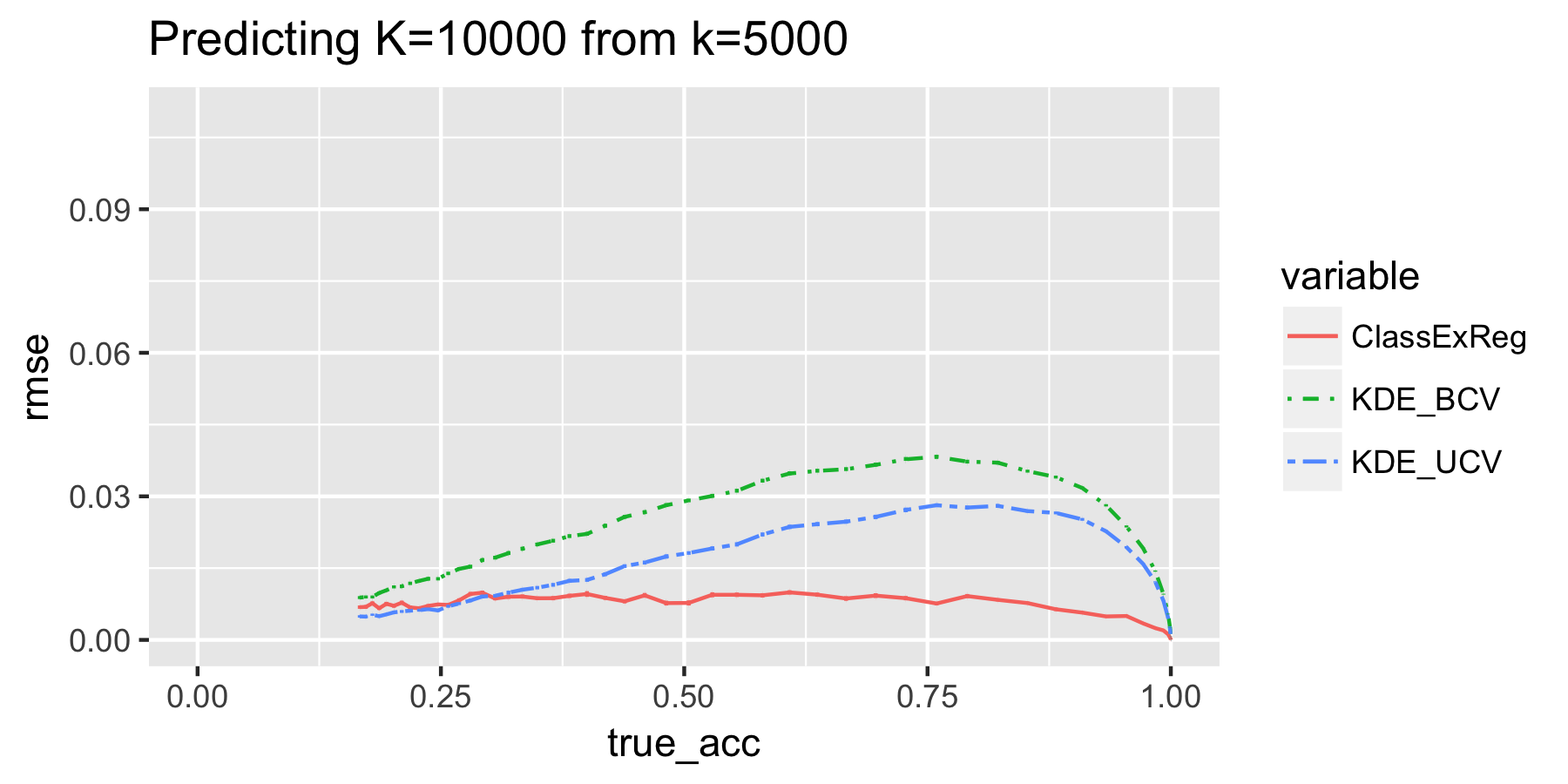} &
\includegraphics[scale = 0.5, clip = true, trim = 0 0 0 0.45in]{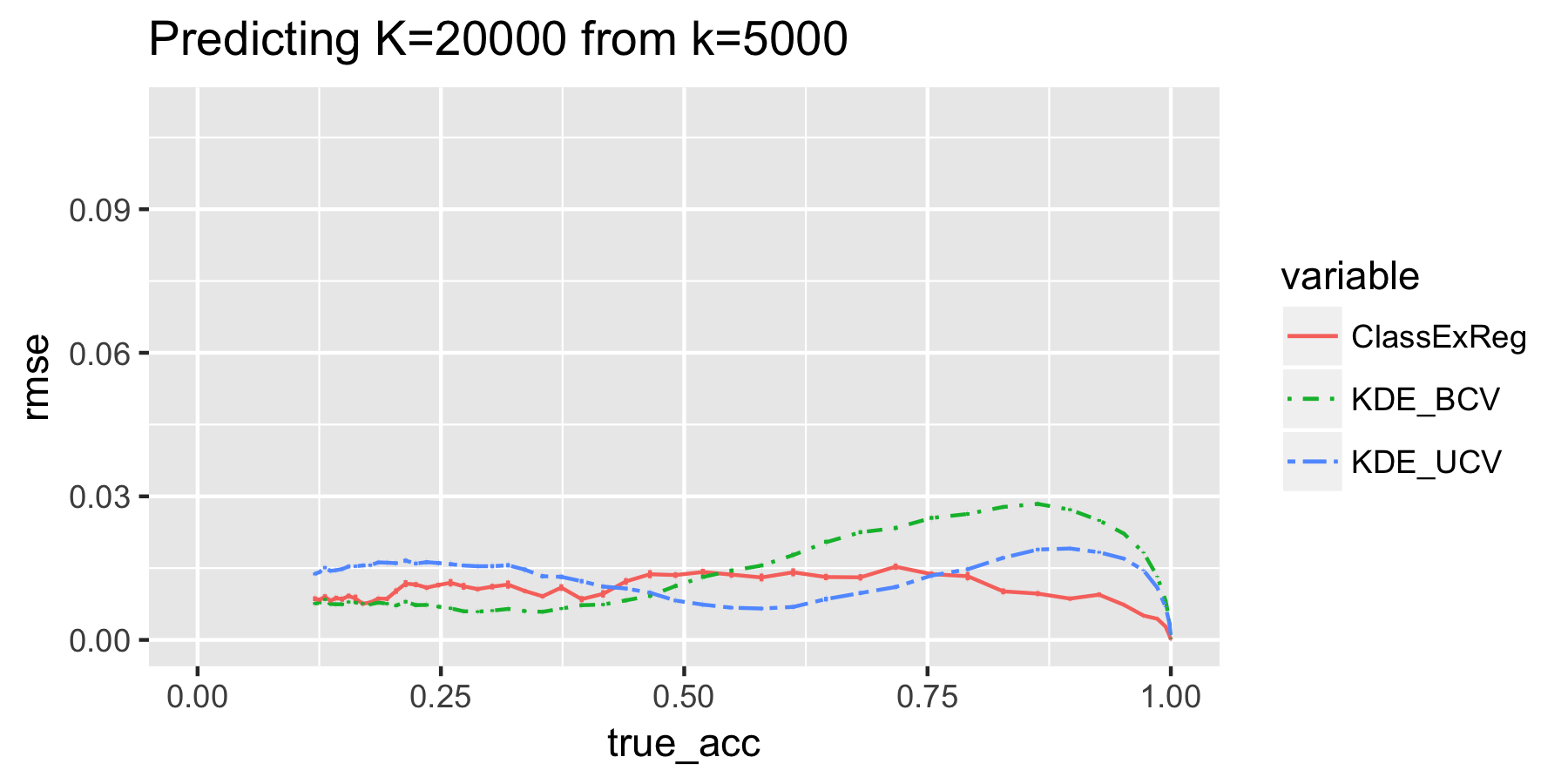}\\
\begin{myfont}Predicting $\text{AGA}_{50000}$\end{myfont} &
\begin{myfont}Predicting $\text{AGA}_{100000}$\end{myfont}\\
\includegraphics[scale = 0.5, clip = true, trim = 0 0 1.25in 0.45in]{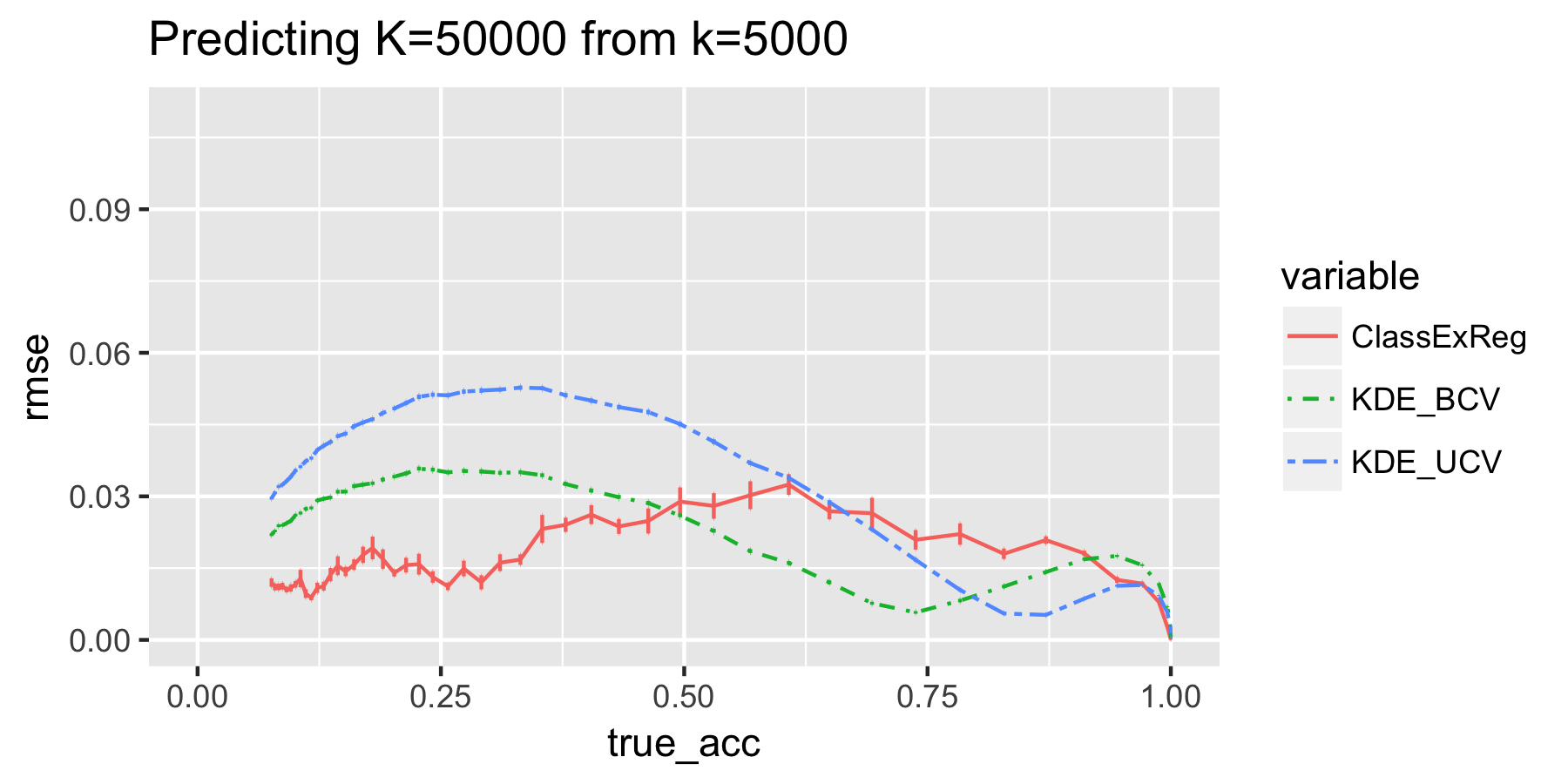} &
\includegraphics[scale = 0.5, clip = true, trim = 0 0 0 0.45in]{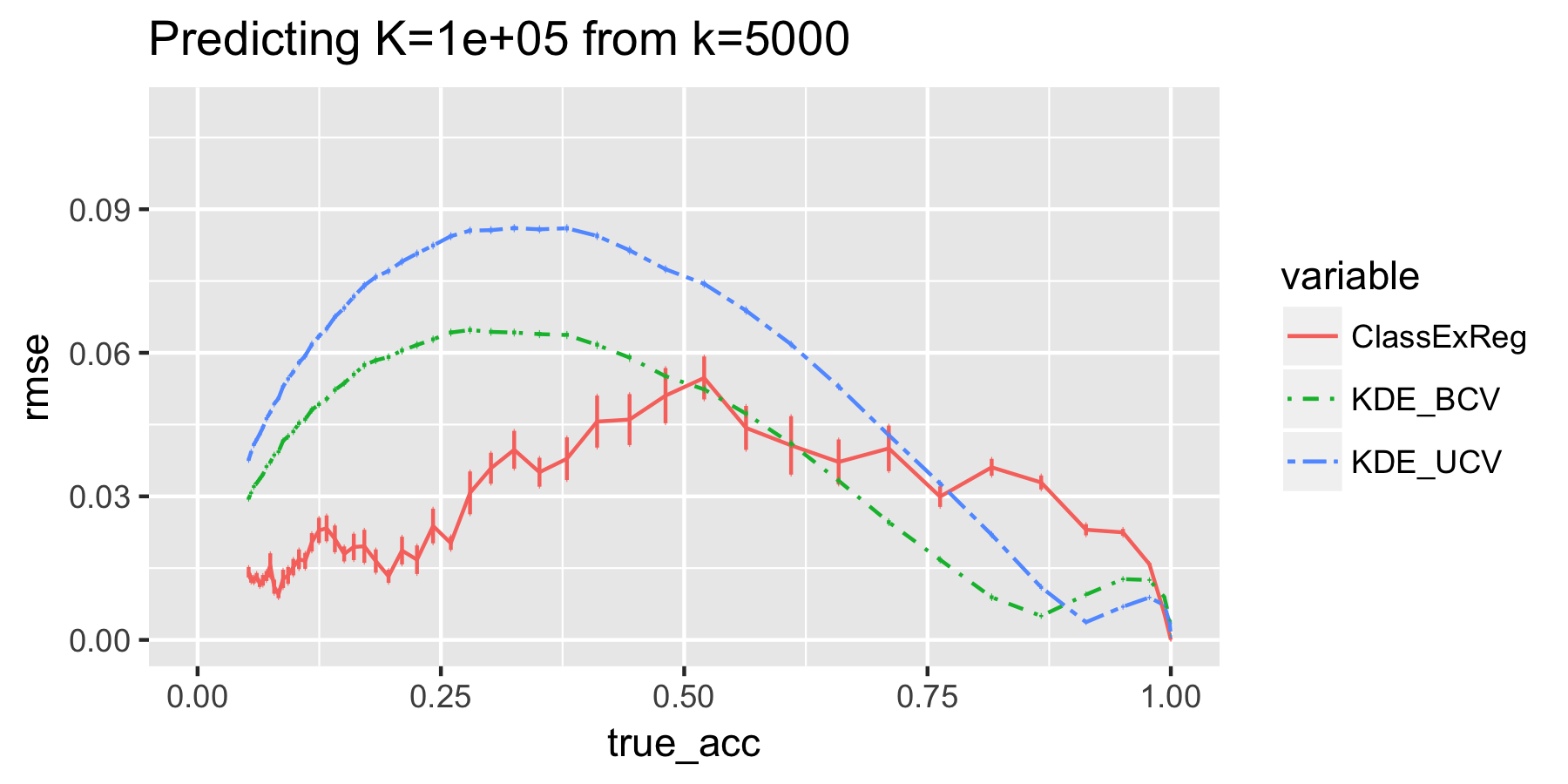}\\ 
\end{tabular}
\caption{\textbf{Simulation results (RMSE):} Simulation study consisting of
  multivariate Gaussian $Y$ with nearest neighbor classifier.
  Prediction RMSE vs true $k_2$-class accuracy for ClassExReg with radial basis
  (\textsf{ClassExReg}), KDE-based methods with biased cross-validation
  (\textsf{KDE\_BCV}) and unbiased cross-validation (\textsf{KDE\_UCV}).}
\label{fig:sim_study}
\end{figure}

\begin{figure}[p]
\centering
\begin{tabular}{ccc}
\multicolumn{3}{c}{\begin{myfont}Extrapolating from $k_1 = 500$\end{myfont}}\\
\begin{myfont}Predicting $\text{AGA}_{2000}$\end{myfont} &
\begin{myfont}Predicting $\text{AGA}_{5000}$\end{myfont} &
\begin{myfont}Predicting $\text{AGA}_{10000}$\end{myfont}\\
\includegraphics[scale = 0.45, clip = true, trim = .22in 0 1.23in 0.4in]{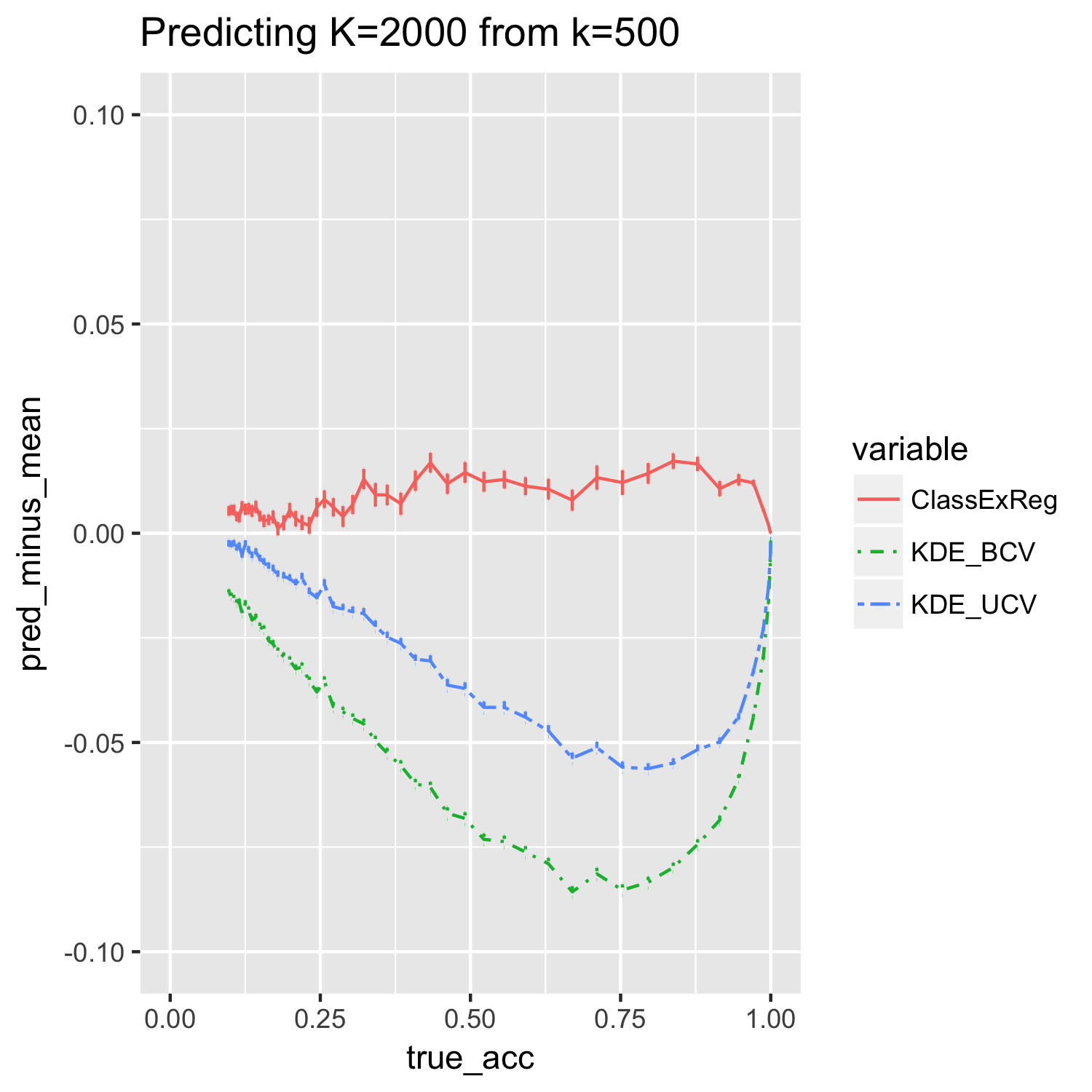} &
\includegraphics[scale = 0.45, clip = true, trim = .3in 0 1.23in 0.4in]{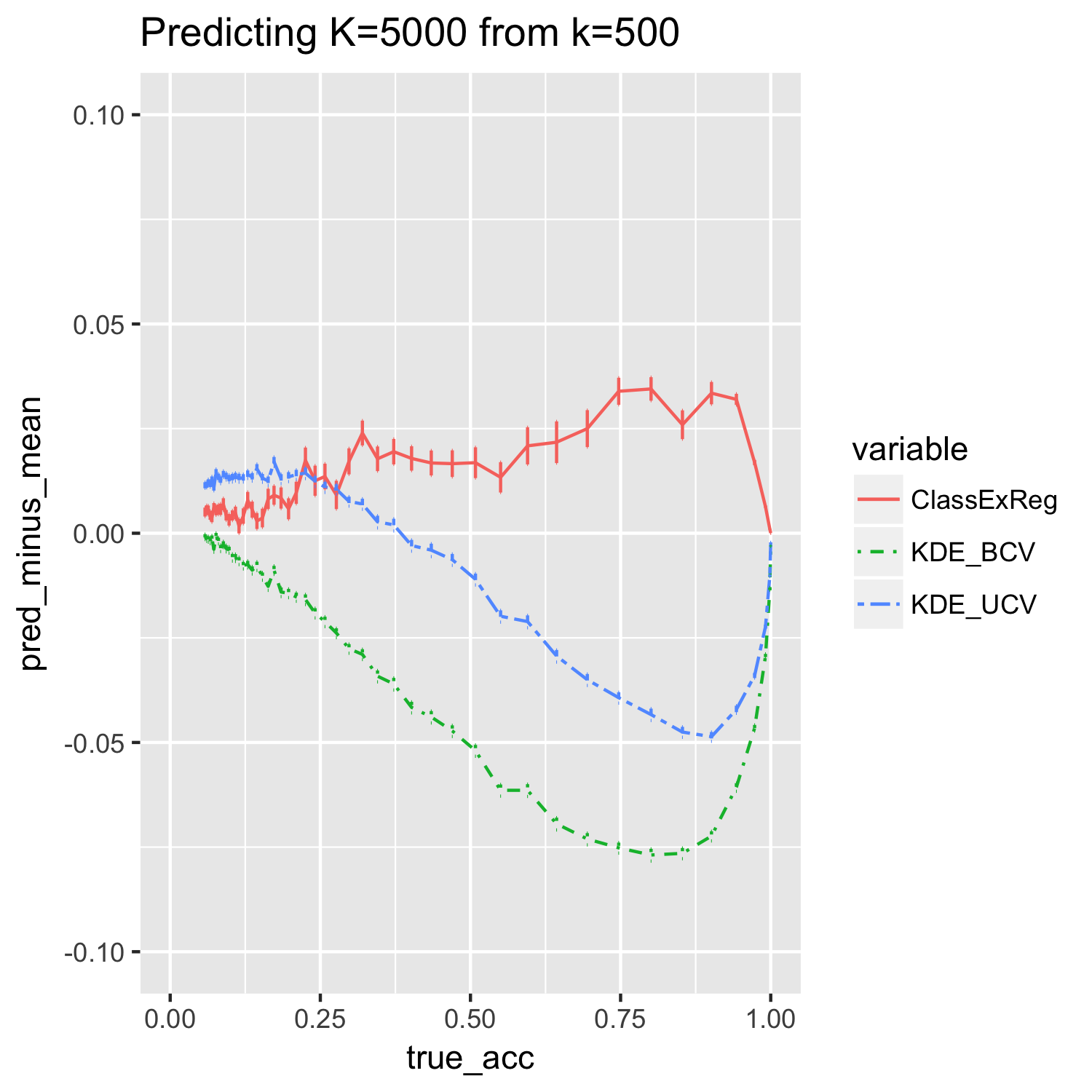} &
\includegraphics[scale = 0.45, clip = true, trim = .3in 0 0.00in 0.4in]{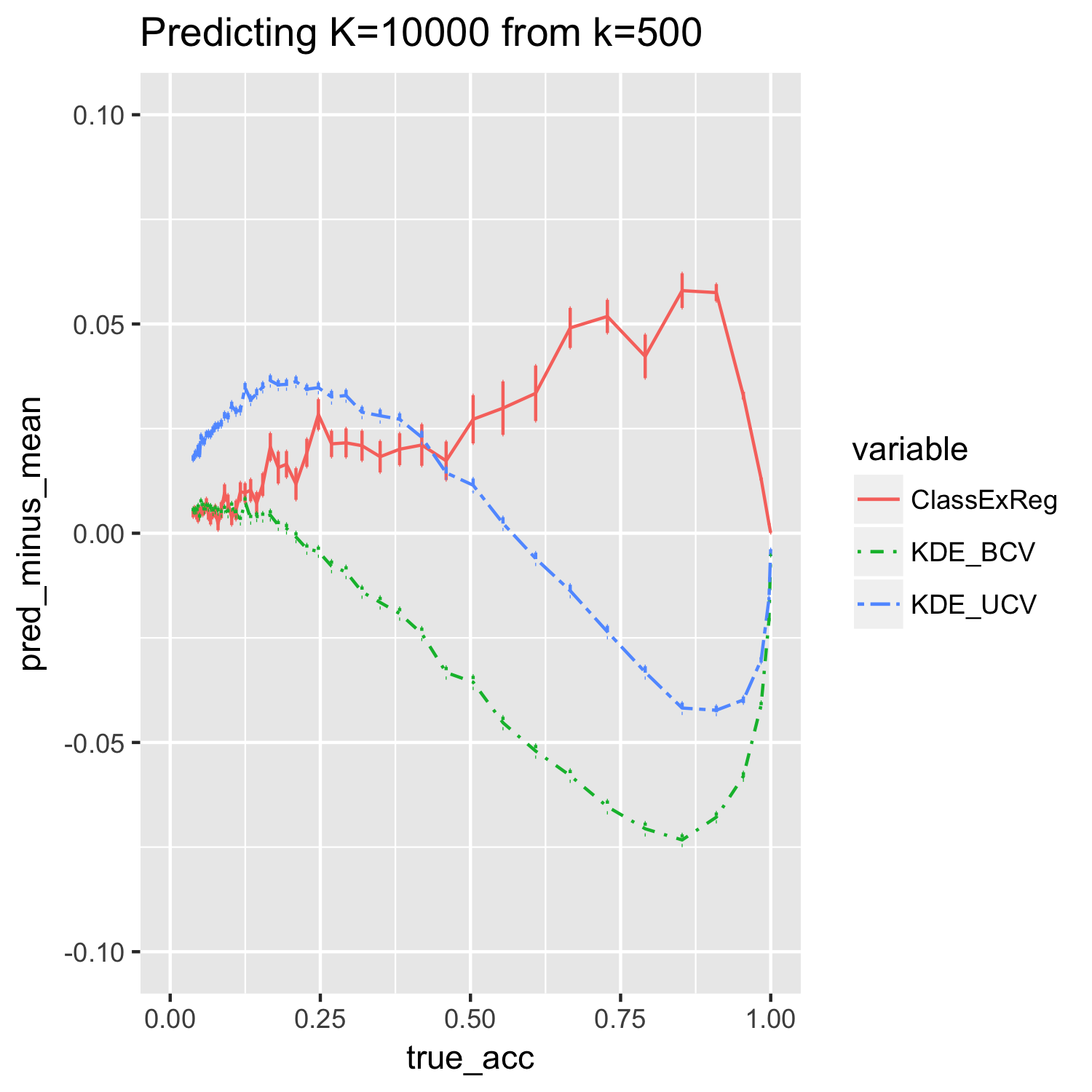}\\ 
\multicolumn{3}{c}{\begin{myfont}Extrapolating from $k_1 = 5000$\end{myfont}}\\
\begin{myfont}Predicting $\text{AGA}_{20000}$\end{myfont} &
\begin{myfont}Predicting $\text{AGA}_{50000}$\end{myfont} &
\begin{myfont}Predicting $\text{AGA}_{100000}$\end{myfont}\\
\includegraphics[scale = 0.45, clip = true, trim = .22in 0 1.23in 0.4in]{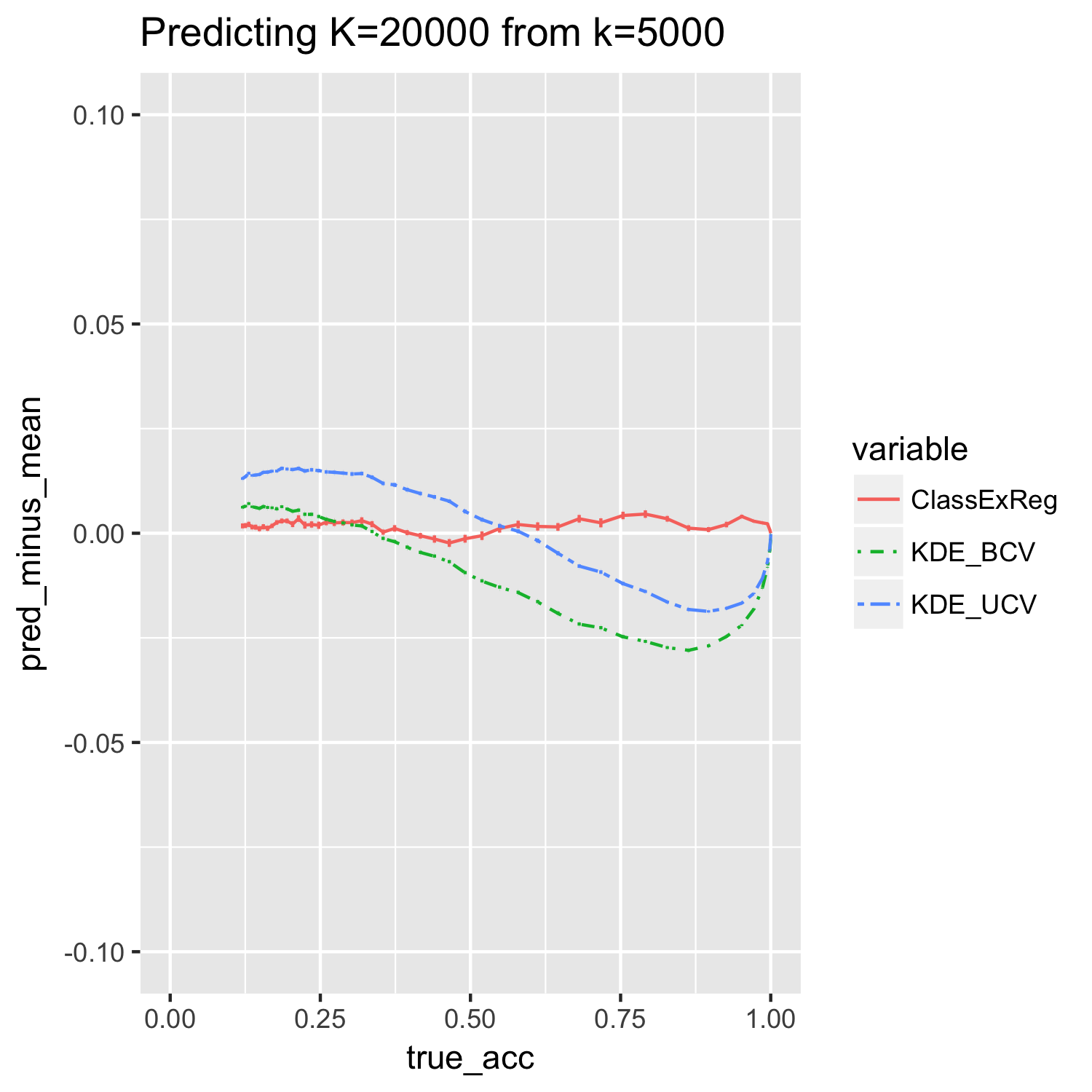} &
\includegraphics[scale = 0.45, clip = true, trim = .3in 0 1.23in 0.4in]{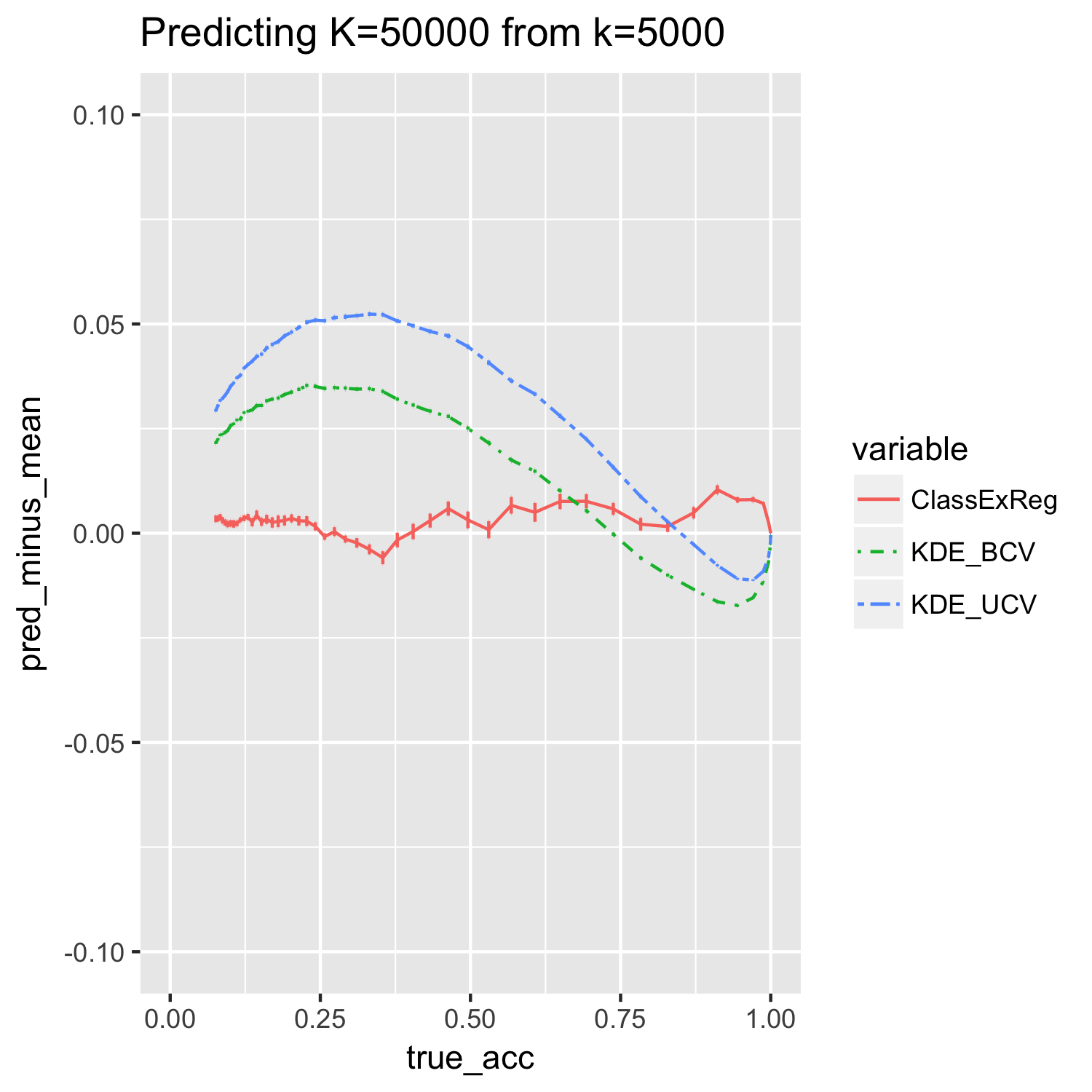} &
\includegraphics[scale = 0.45, clip = true, trim = .3in 0 0.00in 0.4in]{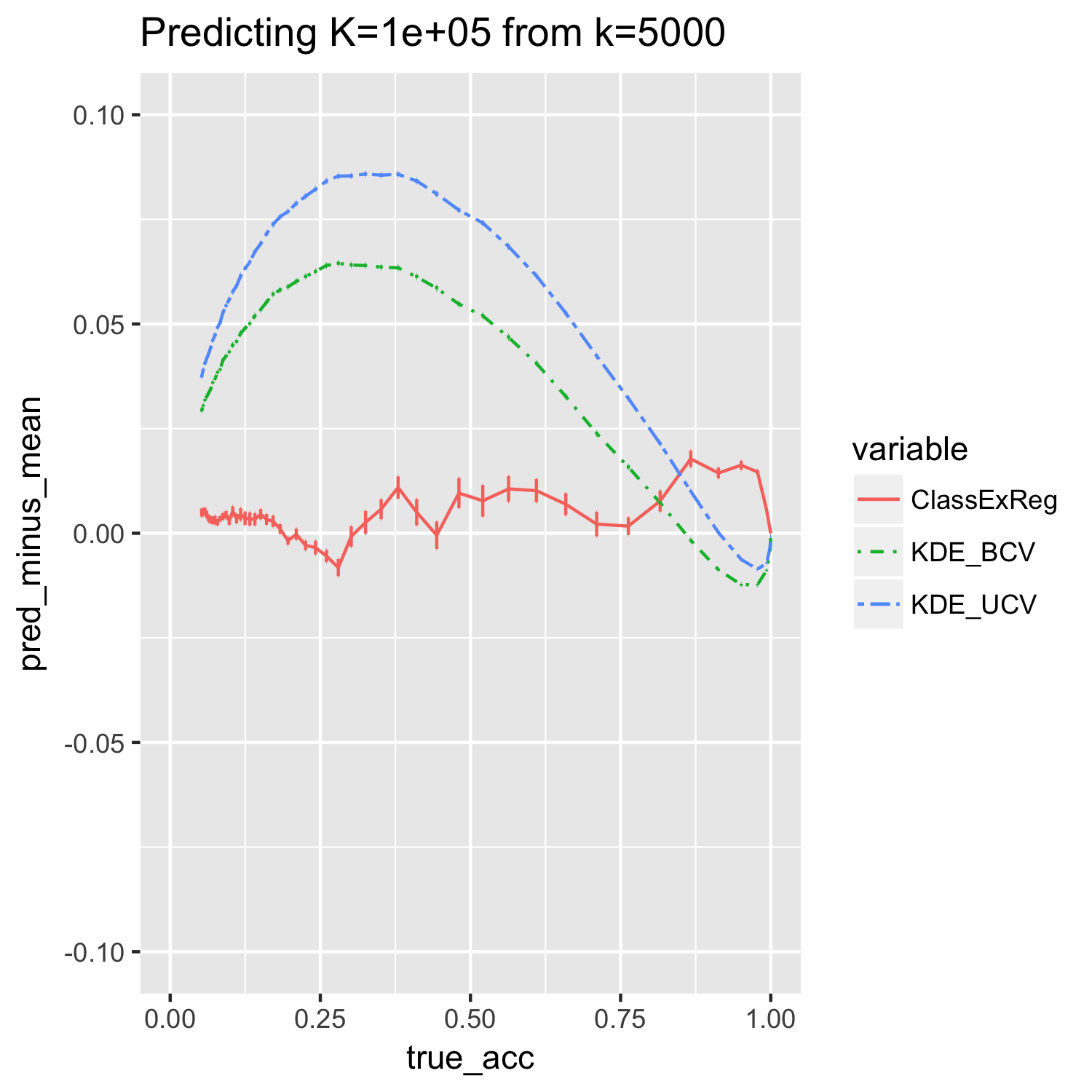}\\ 
\end{tabular}
\caption{\textbf{Simulation results (biases):} Simulation study
  consisting of multivariate Gaussian $Y$ with nearest neighbor
  classifier.  Bias (mean predicted minus true accuracy) vs true $k_2$-class accuracy
  for ClassExReg with radial basis (\textsf{ClassExReg}), KDE-based methods with biased cross-validation (\textsf{KDE\_BCV}) and unbiased cross-validation (\textsf{KDE\_UCV}).}
\label{fig:sim_study_bias}
\end{figure}

\section{Experimental Evaluation}\label{sec:extrapolation_example}

We demonstrate the extrapolation of average accuracy in two data examples:
(i) predicting the
accuracy of a face recognition on a large set of labels from the
system's accuracy on a smaller subset, and (ii) extrapolating the performance of various classifiers on an optical character recognition (OCR) problem in the Telugu script, which has over 400 glyphs.

The face-recognition example takes data from the ``Labeled Faces in the Wild'' data set (\cite{LFWTech}), where we selected the 1672 individuals with at least 2 face photos.  We form a
data set consisting of photo-label pairs $(x_j^{(i)}, y^{(i)})$
for $i = 1,\hdots, 1672$ and $j = 1,2$ by randomly selecting 2 face
photos for each individual. 
We used the OpenFace (\cite{amos2016openface}) embedding for feature
extraction.\footnote{For each photo $x$, a 128-dimensional feature vector
$g(x)$ is obtained as follows.  The computer vision library DLLib is
used to detect landmarks in $x$, and to apply a nonlinear
transformation to align $x$ to a template.  The aligned photograph is
then downsampled to a $96 \times 96$ image. The downsampled image is
fed into a pre-trained deep convolutional neural network to obtain the
128-dimensional feature vector $g(x)$. More details are found in
\cite{amos2016openface}.}
In order to identify a new photo $x^*$, we obtain the feature
vector $g(x^*)$ from the OpenFace network, and guess the label $\hat{y}$
with the minimal Euclidean distance between $g(y^{(i)})$ and $g(x^*)$,
which implies a score function
\[
M_{y^{(i)}}(x^*) = -||g(x_1^{(i)}) - g(x^*)||^2.
\]
In this way, we can compute the test accuracy on all 1672 classes, $\text{TA}_{1672}$, but we also subsample $k_1 = \{100,200,400\}$ classes in order to extrapolate from $k_1$ to 1672 classes.

\begin{figure}[t]
\centering
\begin{tabular}{|c|ccc|c|}
\hline
Label & & Training & & Test\\ \hline
$y^{(1)}$=Amelia & 
  $x_1^{(1)} = $\includegraphics[scale = 0.2]{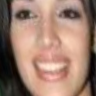} &  
  $x_2^{(1)} = $\includegraphics[scale = 0.2]{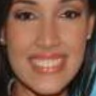} &  
  $x_3^{(1)} = $\includegraphics[scale = 0.2]{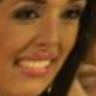} &  
  $x_*^{(1)} = $\includegraphics[scale = 0.2]{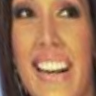} \\ \hline
$y^{(2)}$=Jean-Pierre & 
  $x_1^{(2)} = $\includegraphics[scale = 0.2]{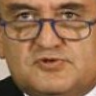} &  
  $x_2^{(2)} = $\includegraphics[scale = 0.2]{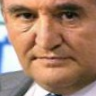} &  
  $x_3^{(2)} = $\includegraphics[scale = 0.2]{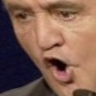} &  
  $x_*^{(2)} = $\includegraphics[scale = 0.2]{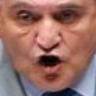} \\ \hline
$y^{(3)}$=Liza & 
  $x_1^{(3)} = $\includegraphics[scale = 0.2]{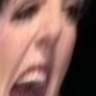} &  
  $x_2^{(3)} = $\includegraphics[scale = 0.2]{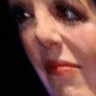} &  
  $x_3^{(3)} = $\includegraphics[scale = 0.2]{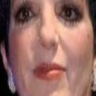} &  
  $x_4^{(3)} = $\includegraphics[scale = 0.2]{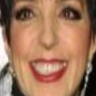} \\ \hline
$y^{(4)}$=Patricia & 
  $x_1^{(4)} = $\includegraphics[scale = 0.2]{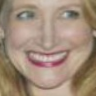} &  
  $x_2^{(4)} = $\includegraphics[scale = 0.2]{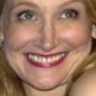} &  
  $x_3^{(4)} = $\includegraphics[scale = 0.2]{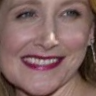} &  
  $x_4^{(4)} = $\includegraphics[scale = 0.2]{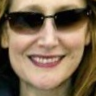} \\ \hline
\end{tabular}
\includegraphics[scale=0.3]{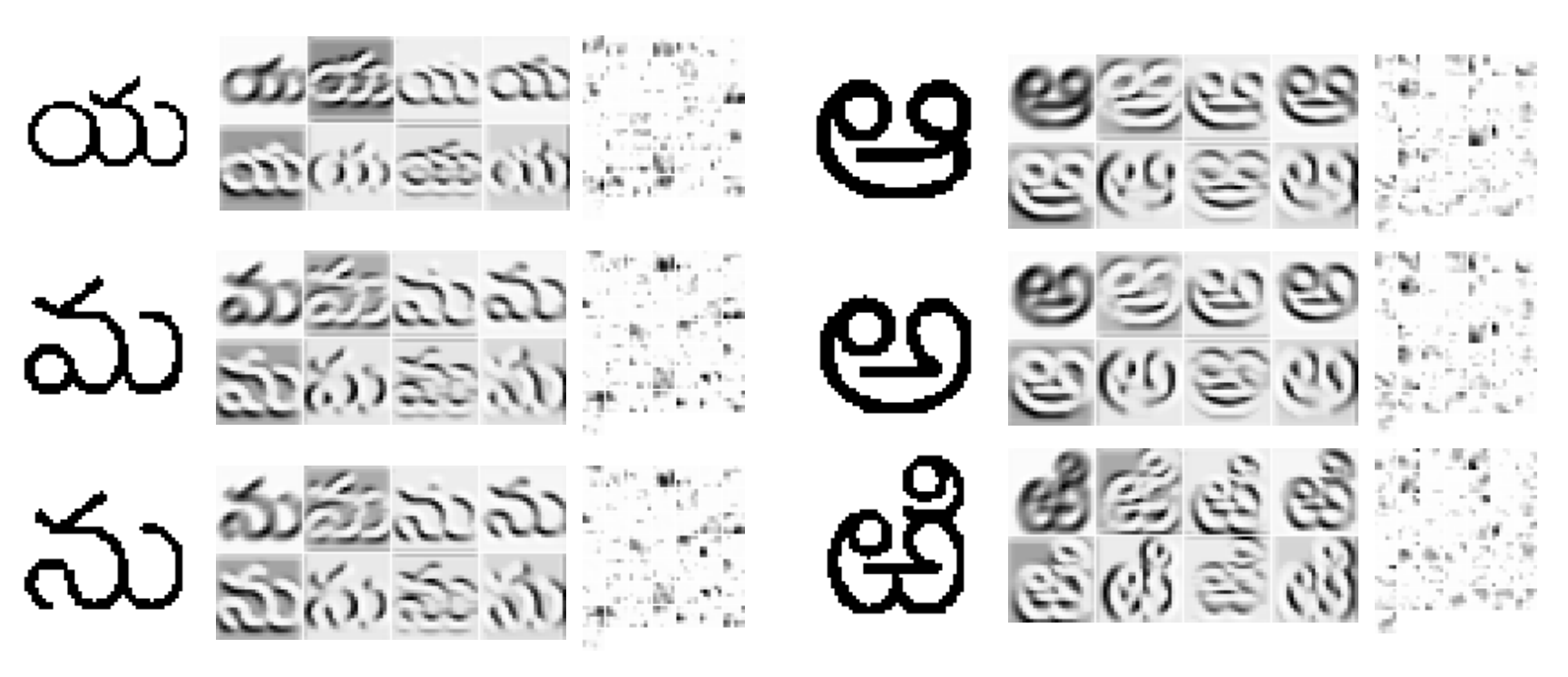}
\caption{\textbf{Face recognition setup (top):} Examples of labels and features from the \emph{Labeled Faces in the Wild} data set. \textbf{Telugu OCR (bottom)}: exemplars from six of the glyph classes, along with intermediate features and final transformations from the deep convolutional network.}
\label{fig:face_rec}
\end{figure}

In the Telugu optical character recognition example
(\cite{achanta2015telugu}), we consider
the use of three different classifiers: logistic
regression, linear support-vector machine (SVM), and a
deep convolutional neural network.\footnote{The network architecture is
  as follows: {\tt
    48x48-4C3-MP2-6C3-8C3-MP2-32C3-50C3-MP2-200C3-SM.}} 
The full data consists of 400 classes with 50 training and 50 test observations for each class. We create a nested hierarchy of subsampled data sets consisting of (i) a subset of 100 classes uniformly sampled without replacement from the 400 classes, and (ii) a subset consisting of 20 classes uniformly sampled without replacement from the size-100 subsample.   We therefore study three different prediction extrapolation problems:
\begin{enumerate}
\item Predicting the accuracy on $k_2 = 100$ classes from $k_1 = 20$ classes, comparing the predicted accuracy to the test accuracy of the classifier on the 100-class subsample as ground truth.
\item Same as (1), but setting $k_2 = 400$ and $k_1 = 20$, and using the full data set for the ground truth.
\item Same as (2), but setting $k_2 = 400$ and $k_1 = 100$.
\end{enumerate}
Note that unlike in the case of the face recognition example,
here the assumption of marginal classification is satisfied for none of
the classifiers.  
We compare the result of our model to the ground
truth obtained by using the full data set.

\subsection{Results}

The extrapolation results for the face recognition problem can be seen in Figure
\ref{fig:lfw_extrapolation2}, which plots the extrapolated accuracy
curves for each method for 100 different subsamples of size $k_1$.  As
can be seen, for all three methods, the variances decrease rapidly as
$k_1$ increases. 

The root-mean-square errors between at $k_2=1672$ can be seen in Table
\ref{tab:lfw_accuracy}.  KDE-BCV achieves the best extrapolation for
all three cases $k_1= \{100,200,400\}$ with KDE-UCV consistently
achieving second place.  These results differ from the ranking of the
RMSEs for the analagous simulation when predicting $k_2 = 2000$ from
$k_1 = 500$ for accuracies around 0.45: in the first row and second
column of Figure \ref{fig:sim_study}, where the true accuracy is 0.43
(from setting $\sigma^2=0.2$), the lowest RMSE belongs to KDE-UCV
(RMSE=$0.0361 \pm 0.001$), followed closely by ClassExReg (RMSE=$0.0372
\pm 0.002$), and KDE-BCV (RMSE=$0.0635\pm0.001$) having the highest
RMSE.  These discrepancies could be explained by differences between
the data distributions between the simulation and the face recognition
example, and also by the fact that we only have access to the $k_2 =
1672$-class ground truth for the real data example.

\begin{figure}[t]
\centering
\begin{tabular}{cccc}
&
\begin{myfont}ClassExReg\end{myfont} & 
\begin{myfont}KDE-BCV\end{myfont} &
\begin{myfont}KDE-UCV\end{myfont}\\
\begin{myfont}$k_1 = 100$\end{myfont} & 
\includegraphics[scale = 0.2, clip = true, trim = 0 0 0 0.6in, valign=c]{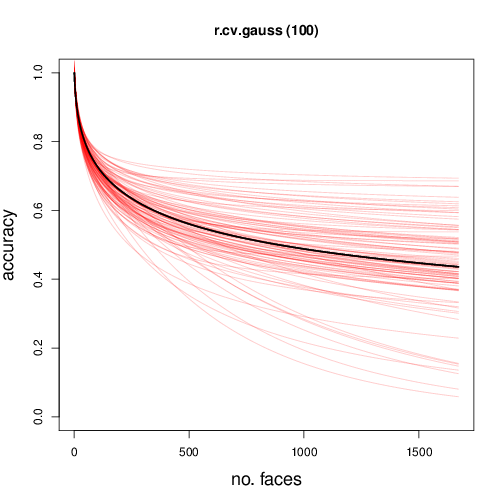} &
\includegraphics[scale = 0.2, clip = true, trim = 0 0 0 0.6in, valign=c]{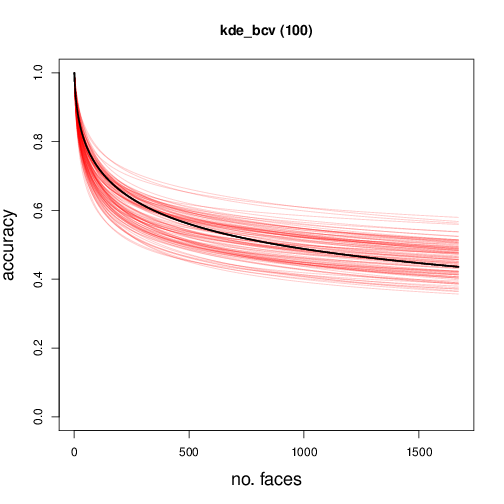} &
\includegraphics[scale = 0.2, clip = true, trim = 0 0 0 0.6in, valign=c]{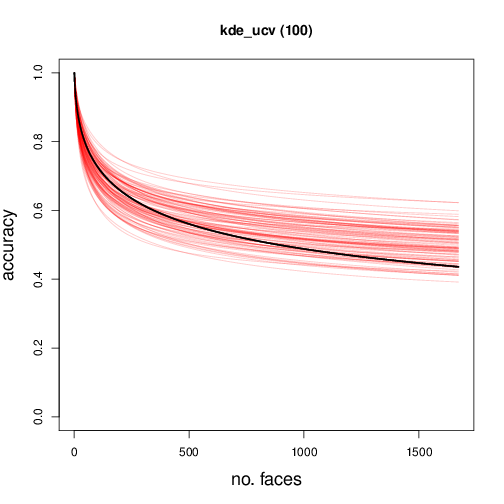} \\
\begin{myfont}$k_1 = 200$\end{myfont} & 
\includegraphics[scale = 0.2, clip = true, trim = 0 0 0 0.6in, valign=c]{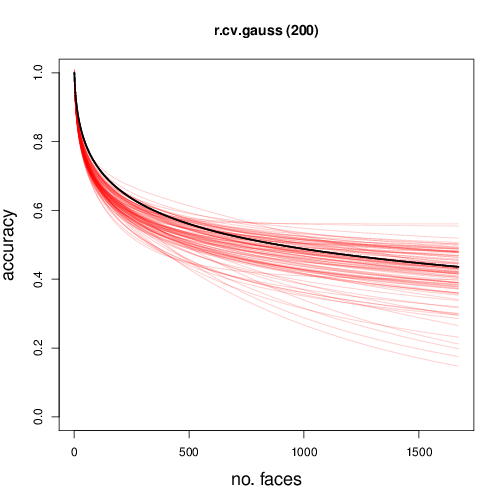} &
\includegraphics[scale = 0.2, clip = true, trim = 0 0 0 0.6in, valign=c]{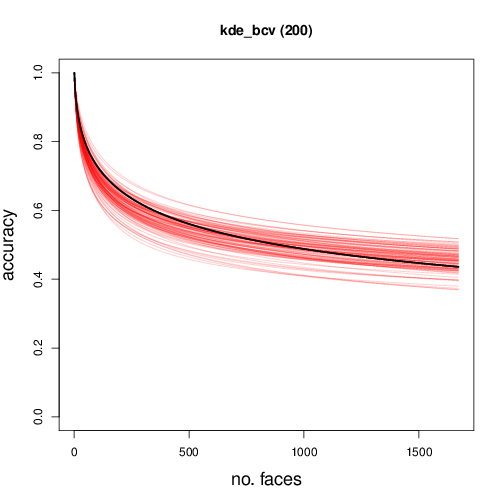} &
\includegraphics[scale = 0.2, clip = true, trim = 0 0 0 0.6in, valign=c]{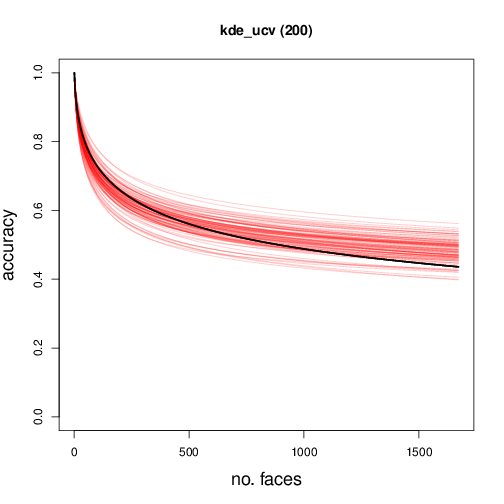} \\
\begin{myfont}$k_1 = 400$\end{myfont} & 
\includegraphics[scale = 0.2, clip = true, trim = 0 0 0 0.6in, valign=c]{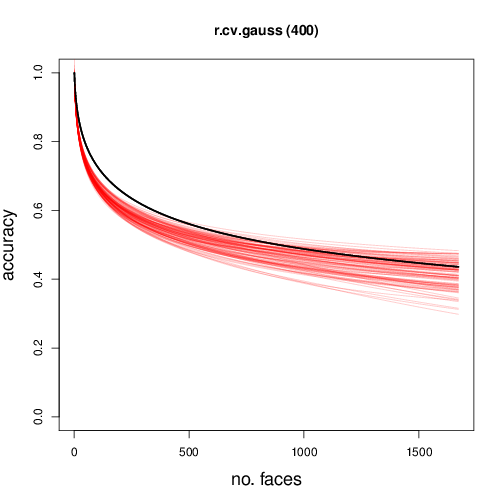} &
\includegraphics[scale = 0.2, clip = true, trim = 0 0 0 0.6in, valign=c]{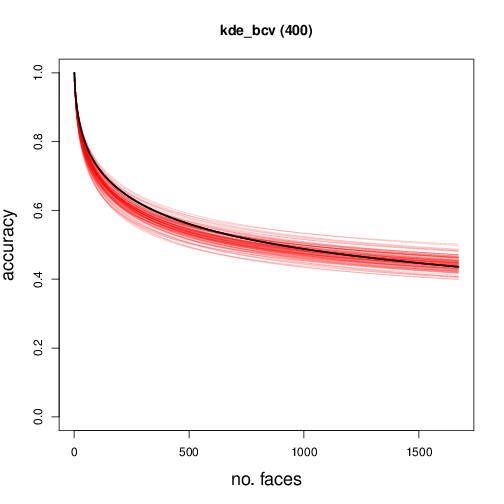} &
\includegraphics[scale = 0.2, clip = true, trim = 0 0 0 0.6in, valign=c]{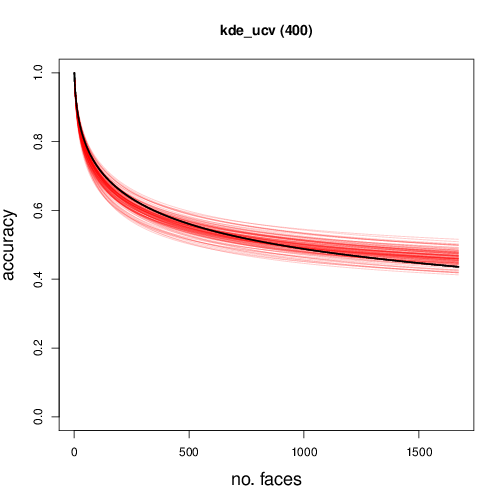} \\
\end{tabular}
\caption{\textbf{Predicted accuracy curves for face-recognition example}: The
  plots show predicted accuracies. Each red curve represents the predicted accuracies using
  a single subsample of size $k_1$. The black curve shows the average test accuracy obtained from the full data set. }
\label{fig:lfw_extrapolation2}
\end{figure}

\begin{table}[t]
\centering
\begin{tabular}{c||c|c|c|c}
\hline
$k_1$ & ClassExReg & KDE-BCV & KDE-UCV \\\hline
100 & 0.113 (0.002) & \textbf{0.053} (0.001) & 0.082 (0.001) \\\hline
200 & 0.058 (0.002)& \textbf{0.037} (0.001)  & 0.057 (0.001) \\ \hline
400 & 0.050 (0.001) & \textbf{0.024} (0.001)& 0.035 (0.001) \\\hline
\end{tabular}
\caption{\textbf{Face-recognition extrapolation RMSEs}: RMSE (se) on predicting $\text{TA}_{1672}$ from $k_1$ classes}\label{tab:lfw_accuracy}
\end{table}

The results for Telugu OCR classification are displayed in Table \ref{tab:tel_accuracy}.  If we rank the three extrapolation methods in terms of distance to the ground truth accuracy, we see a consistent pattern of rankings between the 20-to-100 extrapolation and the 100-to-400 extrapolation.  As we remarked in the simulation, the difficulty of extrapolation appears to be primarily sensitive to the extrapolation ratio $\frac{k_2}{k_1}$, which are similar (5 versus 4) in the 20-to-100 and 100-to-400 problems.  In both settings, ClassExReg comes closest to the ground truth for the Deep CNN and the SVM, but KDE-BCV comes closest to ground truth for the Logistic regression.  However, even for logistic regression, ClassExReg does better or comparably to KDE-UCV.

In the 20-to-400 extrapolation, which has the highest extrapolation ratio ($\frac{k_2}{k_1} = 20$), none of the three extrapolation methods performs consistently well for all three classifiers.  It could be the case that the variability is a dominating effect given the small training set, making it difficult to compare extrapolation methods using the 20-to-400 extrapolation task.

Unlike in the face recognition example, we did not resample training classes here, because that would require retraining all of the classifiers--which would be prohibitively time-consuming for the Deep CNN.  Thus, we cannot comment on the robustness of the comparisons from this example, though it is likely that we would obtain different rankings under a new resampling of the training classes.

\begin{table}[t]
\centering
\begin{tabular}{c|c||c|c||c|c|c}
$k_1$ & $k_2$ & Classifier & True & ClassExReg & KDE-BCV & KDE-UCV \\ \hline
 20 & 100 & Deep CNN & 0.9908 & \textbf{0.9905} & 0.7138 & 0.6507 \\ 
    &     & Logistic & 0.8490 & 0.8980 & \textbf{0.8414} & 0.8161 \\
    &     & SVM      & 0.7582 & \textbf{0.8192} & 0.6544 & 0.5771 \\ \hline
 20 & 400 & Deep CNN & 0.9860 & \textbf{0.9614} & 0.4903 & 0.3863 \\
    &     & Logistic & 0.7107 & 0.8824 & 0.7467 & \textbf{0.7015} \\
    &     & SVM      & 0.5452 & 0.6725 & \textbf{0.5163} & 0.4070 \\ \hline
100 & 400 & Deep CNN & 0.9860 & \textbf{0.9837}& 0.8910 & 0.8625  \\ 
    &     & Logistic & 0.7107 & 0.7214& \textbf{0.7089} & 0.6776  \\
    &     & SVM      & 0.5452 & \textbf{0.5969}& 0.4369 & 0.3528  \\ 
\hline
\end{tabular}
\caption{\textbf{Telugu OCR extrapolated accuracies}: Extrapolating from $k_1$ to $k_2$ classes in Telugu OCR for three different classifiers: logistic regression, support vector machine, and deep convolutional network}\label{tab:tel_accuracy}
\end{table}

\section{Discussion}
\label{sec:discussion}
In this work, we suggest treating the class set in a classification
task as random, in order to extrapolate classification performance on
a small task to the expected performance on a larger unobserved task.
We show that average generalized accuracy decreases with increased
label set size like the $(k-1)$th moment of a distribution function.
Furthermore, we introduce an algorithm for estimating this underlying
distribution, that allows efficient computation of higher order
moments. Code for the methods and the simulations can be found in \url{https://github.com/snarles/ClassEx}.

There are many choices and simplifying assumptions used in
the description of the method.  In this discussion, we discuss these
decisions and map some alternative models or strategies for future
work.

%\noindent\textsl{Sampling}

Since our analysis is currently restricted to i.i.d. sampling of classes, 
one direction for future work is to generalize the sampling mechanism,
such as to cluster sampling.  More broadly, the assumption that the labels in $\mathcal{S}_k$ are a
random sample from a homogeneous distribution $\pi$ may be inappropriate.  Many
natural classification problems arise from hierarchically partitioning
a space of instances into a set of labels.  Therefore, rather than
modeling $\mathcal{S}_k$ as a random sample, it may be more suitable
to model it as a random hierarchical partition of $\mathcal{Y}$, such as one arising from an optional P{\'o}lya tree process
\citep{wong2010optional}.
Finally, note that we assume no knowledge about the new class-set
except for its size. Better accuracy might be achieved if some partial information 
is known.

%\noindent\textsl{Arbitrary cost functions}

Also, we only discussed extrapolating the
classification accuracy--or equivalently, the risk for the zero-one
cost function.  However, it is possible to extend our analysis to risk
functions with arbitrary cost functions, which is the subject of
forthcoming work.

%\noindent\textsl{Parametric models}

A third direction of exploration is impose additional modeling assumptions for specific problems.  ClassExReg adopts a non-parametric model of the discriminability function $D(u)$, in the sense that $D(u)$ was defined via a spline expansion.   However, an alternative approach is to assume a parametric family for $D(u)$ defined by a small number of parameters.  In forthcoming work, we show that under certain limiting conditions, $D(u)$ is well-described by a two-parameter family.  This substantially increases the efficiency of estimation in cases where the limiting conditions are well-approximated.

\section*{Acknowledgments}

We thank Jonathan Taylor, Trevor Hastie, John Duchi, Steve
  Mussmann, Qingyun Sun, Robert Tibshirani, Patrick McClure, and Gal Elidan for useful discussion.  CZ
  is supported by an NSF graduate research fellowship, and would also
  like to thank the European Research Council under the ERC grant agreement $\text{n}^\circ$[PSARPS-294519]  for travel support.

\bibliography{example}

\end{document}